\documentclass[12pt]{article}

\usepackage{jmlr2eh}

\usepackage{algorithmic,algorithm}

\usepackage{wrapfig}
\usepackage{graphicx}
\usepackage{subfigure}
\usepackage{epsfig}
\usepackage{tabularx}
\usepackage{fancyvrb}
\usepackage{amsmath}
\usepackage{amssymb}
\usepackage{textcomp}
\usepackage[english]{babel}
\usepackage{epstopdf}
\usepackage[bookmarks=true]{hyperref}

\usepackage[usenames,dvipsnames]{pstricks}
\usepackage{pst-grad} % For gradients
\usepackage{pst-plot} % For axes

\newcommand{\B}{\begin{equation}}
\newcommand{\E}{\end{equation}}

\def\mc#1{\mathcal{#1}}

\def\mrm#1{\mathrm{#1}}
\def\mbi#1{\boldsymbol{#1}} % Bold and italic (math bold italic)
\def\v#1{\mbi{#1}} % Vector notation
\def\s#1{\mc{#1}} % Set notation
\def\lll{{\ell}}
\def\nll{\lll}
\def\uu{\mrm{u}_0}
\def\capping{\theta}

\def\<{\left\langle} % Angle brackets
\def\>{\right\rangle}

\def\reals{\mathbb{R}} % Real number symbol
\def\R{\mathbb{R}}
\def\half{\frac{1}{2}}

\newcommand{\simplex}{\Delta}
\newcommand{\numchange}{\kappa}
\newcommand{\eqdef}{\stackrel{\mbox{\tiny def}}{=}}

\title{The Maximum Entropy Relaxation Path}
\author{
\hspace{-0.3cm}
\name Moshe Dubiner \email moshe@google.com\\
\addr Google Inc. \\
1600 Amphitheater Parkway \\
Mountain View, CA 94043, USA
\AND
Matan Gavish \email gavish@stanford.edu\\
\addr Department of Statistics\\
Stanford University\\
Stanford, CA 94305, USA
\AND
\name Yoram Singer \email singer@google.com\\
\addr Google Inc. \\
1600 Amphitheater Parkway \\
Mountain View, CA 94043, USA}

\begin{document}

\maketitle

\begin{abstract}
The relaxed maximum entropy problem is concerned with finding a probability
distribution on a finite set that minimizes the relative entropy to a given
prior distribution, while satisfying relaxed max-norm constraints with
respect to a third observed multinomial distribution. We study the entire
relaxation path for this problem in detail. We show existence and a geometric
description of the relaxation path. Specifically, we show that the maximum
entropy relaxation path admits a planar geometric description as an increasing,
piecewise linear function in the inverse relaxation parameter.  We derive fast
algorithms for tracking the path. In various realistic settings, our algorithms
require $O(n\log(n))$ operations for probability distributions on $n$ points,
making it possible to handle large problems.  Once the path has been recovered,
we show that given a validation set, the family of admissible models is reduced
from an infinite family to a small, discrete set.  We demonstrate the merits of
our approach in experiments with synthetic data and discuss its potential for
the estimation of compact n-gram language models.
%\footnote{A short version of this
%paper titled ``Entire Relaxation Path for Maximum Entropy Problems''
%by M. Dubiner and Y. Singer appeared at EMNLP'11.}.
%
\end{abstract}
\newpage
\section{Introduction} \label{intro:sec}
This paper studies the set of solutions to the relaxed maximum entropy problem
\begin{eqnarray} \label{main:eqn}
   & \displaystyle \min_{\v{p}\,\in\simplex} &
    \sum_{j=1}^n p_j\log\!\left(\frac{p_j}{u_j}\right) \\
   & \hspace{1cm} \mbox{ s.t. } & \|\v{p}-\v{q}\|_\infty \leq 1/\nu \,,
\nonumber
\end{eqnarray}
where $\Delta$ is the probability simplex in $\R^n$, and where $ \v{q},\v{u}\in\simplex$ are
given. The solution set, indexed by the relaxation parameter $\nu\geq 0$,
is known as the {\em relaxation path} of \eqref{main:eqn}.

Numerous machine learning tasks are cast as an optimization problem, similar
to the form above, in which the objective decomposes into
an empirical risk function, and an added regularization term, which controls the model complexity. 
The tradeoff between risk and model complexity is determined by a scalar relaxation
parameter, which is often tuned by solving the optimization problem
for multiple relaxation parameter values, and using a validation set in order to choose the
most appropriate one. An alternative approach is to characterize the solution
for {\em any} possible relaxation parameter, effectively solving the problem for all
values simultaneously. This characterization is known in the literature as a
solution for the entire regularization (or relaxation) path.

Characterization of the entire relaxation path of specific problems has been the focus of a
relatively small number of research papers. Concretely, \cite{OsbornePrTu00}
and \cite{EfronHaJoTi04} provide two different characterizations of the entire
relaxation path of the Lasso~\citep{Tibshirani96b}. \cite{OsbornePrTu00b}
called the mapping between the space of regularization values to the set
of solutions {\em Homotopy}, a term that we adopt here. In the context of support
vector machines, \cite{PontilVe98} and \cite{HastieRoTiJi04} characterized the
relaxation path observed for SVM. \cite{RossetZh07} gave a general
characterization for losses which admit a linear relaxation path.
\cite{Rosset05} described an approximate characterization for the relaxation
path of logistic regression and related problems. \cite{ZhaoYu04} provided an
approximate characterization of the relaxation path for any convex loss
function with an additive $\ell_1$ regularization term. \cite{park2007l1}
provided an approximate characterization of the relaxation path for
generalized linear models. More recently, \cite{tibshirani2011solution}
characterized the relaxation path for a generalized Lasso problem where the
$\ell_1$ penalty is applied to a linear transformation of the Lasso variables.

In this paper we study a specific problem, in which the objective term is additively
separable. Throughout most of the paper we focus on the case where the
objective is the relative entropy between two multinomial distributions with a
max-norm constraint, known as the {\em relaxed maximum entropy} problem. The general
form of maximum entropy subject to relaxed constraints was studied
in~\citep{DudikPhSc07}. \citeauthor{DudikPhSc07} described a general account
for relaxed maximum entropy problems and derived corresponding generalization
bounds. For other related relaxation approaches see also Sec.~1.1
in~\citep{DudikPhSc07}.

The paper proceeds as follows.  In Sec.~\ref{setting:sec} we study the general
case of additively separable convex objective.  We show that in this generality,
the relaxation path exists and is given as the solution to an equation in the
inverse relaxation parameter.  In Sec.~\ref{munu:sec} we show that the maximum
entropy relaxation path admits a simple planar geometric characterization in the
inverse relaxation parameter, as an increasing, piecewise linear function.  In
Sec.~\ref{local_homotopy:sec} we build on the geometrical description and derive
a path tracking algorithm with worst-case time complexity  $O(n^3)$ for vectors of length $n$. Specializations of
the general algorithms are provided, which are able in realistic cases to
recover the path in time complexity $O(n\log(n))$, making it possible to handle
very large maximum entropy problems.  In Sec.~\ref{cv:sec} we
describe an efficient cross-validation procedure based on the entire path
solution:  by solving for the relaxation path, given
a validation set, we are able to reduce the family of admissible models, which
are considered for model selection, from an infinite family to a small, discrete
set.  In Sec.~\ref{eval:sec} we illustrate the merits of our approach in
experiments with synthetic data.  Sec.~\ref{ngram:sec} describes an application
of our approach for estimation of compact n-gram language models.  In
Sec.~\ref{extensions:sec} we describe extensions, including a different case
in which the relaxation path exists and can be tracked efficiently.  Finally, in
Appendix~\ref{global_homotopy:sec} we provide a more complicated path tracking
algorithm, with improved worst-case computational time complexity $O(n^2\log(n))$.

\section{The Relaxation Path: Basic Properties} \label{setting:sec}
\subsection{Notation} We denote vectors with bold face letter, e.g. $\v{v}$.
Sums are denoted by calligraphic letters, e.g. $\s{M}=\sum_j m_j$. We use the
shorthand $[n]$ to denote the set of integers $\{1,\dots,n\}$. The
inner-product between two vectors $\v{u}$ and $\v{v}$ is denoted,
$\v{u}\cdot\v{v}$. The {\em generalized} simplex with respect to a vector
$\v{m}$ whose components are positive is $\simplex(\v{m}) = \{\v{p} \, | \,
\v{p}\cdot\v{m} = 1 , \forall j: p_j\geq 0\}$.  Note that when for all $j$,
$m_j=1$, we get the standard definition of the simplex. Similarly, when for all
$j$, $m_j=\frac{1}{R}$, we retrieve the positive part of the $\ell_1$ ball of
radius $R$.  We call $\v{m}$ the multiplicity vector. As the name implies, its
role is to incorporate the case where there are repeated entries in $\v{p}$
which take the same values.  These repeated values are encoded by setting their
multiplicity accordingly. An addition rationale for allowing multiplicity
vectors is given in Sec.~\ref{extensions:sec}.

\subsection{General Solution Characterization}

Consider first the following generalized form of problem~\ref{main:eqn}:
\B \label{gen:eqn}
\min_{\v{p}\in\simplex(\v{m})} \phi(\v{p})  ~ \mbox{ s.t } ~
  \|\v{p}-\v{q}\|_\infty \leq 1 / \nu\,, 
\E
where 
$\v{q},\v{u}\in\simplex(\v{m})$ are given probability vectors,  and
$\phi$ is a strictly convex function. This form is a convenient
apparatus for describing the general properties of the solution.

Clearly, since for any $\nu\ge 0$ the objective function
of~(\ref{gen:eqn}) is a strictly convex function over a compact convex domain,
its optimum $\v{p}(\nu)$ exists. Further, it is unique and can be
viewed as a continuous vector function in~$\nu$.  Let us now further
characterize the form of the solution $\v{p}$. We can partition the set of
indices in $[n]$ into three disjoint sets depending on whether either of
the max-norm constraints $p_j-q_j\leq 1/\nu$ or $p_j-q_j\geq -1/\nu$ is
binding:
\begin{eqnarray}
  I_-(\nu) & = & \{1\le j\le n\ |\ p_j = q_j-1/\nu\} \nonumber \\
  I_0(\nu) & = & \{1\le j\le n\ |\ |p_j - q_j| < 1/\nu\} \label{Isets:eqn} \\
  I_+(\nu) & = & \{1\le j\le n\ |\ p_j = q_j+1/\nu\} \nonumber  ~ ~ .
\end{eqnarray}
An alternative form of representing the partition $(I_-,I_0,I_+)$ is
obtained by associating an indicator value with each coordinate resulting
in a vector $\v{s}\in\{-1,0,1\}^n$ where
$$
  s_j = \left\{\hspace{-0.1cm}\begin{array}{rl}
  -1 & j\in I_-\\
   0 & j\in I_0\\
  +1 & j\in I_+
 \end{array} \right. ~ ~ .
$$
We make use of both notations.
Our goal is to devise an algorithmic infrastructure that lets us
reveal the correct partition without examining all of the $3^n$ possible
partitions. The following
characterizes the gradient of solution in terms of the partition
$(I_-,I_0,I_+)$.
\begin{lemma} \label{gen2:thm}
Let $\v{p}\in\simplex(\v{m})$, $\phi$ be a strictly convex function, and
$\nu>0$. Assume that $\phi$ is differentiable over $\simplex(\v{m})$, and that
its optimum has no zero coordinates.  Let $\partial_j\phi = \frac{\partial
\phi}{\partial p_j}$ denote the $j$-th coordinate of the gradient of $\phi$.
Then, $\v{p}$ minimizes (\ref{gen:eqn}) iff there exists $-\infty<\eta<\infty$
such that for any $1\le j\le n$ exactly one of the following three conditions
holds,
\begin{eqnarray}
  \partial_j\phi(\v{p}) \ge \eta\,m_j & \mbox{if} & p_j=q_j-1/\nu \label{c4}\\
  \partial_j\phi(\v{p})  =  \eta\,m_j & \mbox{if} & |p_j-q_j|<1/\nu \label{c5}\\
  \partial_j\phi(\v{p}) \le \eta\,m_j & \mbox{if} & p_j=q_j+1/\nu \label{c6}
  ~ ~ .
\end{eqnarray}
\end{lemma}

\begin{proof}
We prove the lemma by using of the complementary slackness conditions for
optimality. For brevity we assume throughout the rest of the proof that
$m_j=1$ for all $j\in[n]$. We associate a Lagrange multiplier
$\alpha_j^+\geq 0$ for the constraint $p_j - q_j \leq 1/\nu$ and
$\alpha_j^-\geq 0$ for the constraint $p_j - q_j \geq -1/\nu$. We use $\eta$
to denote the Lagrange multiplier for the simplex constraint $\sum_j p_j=1$.
Since we assumed that the solution is strictly positive, we know that the
Lagrange multipliers corresponding to the positivity constraints of $\v{p}$
would be all zero at min-max saddle point of the Lagrangian. Hence, we get the
following Lagrangian,
$$
{\cal L} = \phi(\v{p}) +
   \sum_{j=1}^m \alpha_j^+ (p_j - q_j - 1/\nu) -
   \sum_{j=1}^m \alpha_j^- (p_j - q_j + 1/\nu) -
   \eta \left(\sum_{j=1}^n p_j - 1\right) ~ ~ .
$$
From the necessary condition for optimality we know that
$\partial {\cal L}/\partial p_j = 0$ for all $j\in[n]$. Therefore, we get
that for all indices $j$ the following holds at the optimum,
\begin{equation} \label{lagrange_zero:eqn}
\partial_j\phi(\v{p}) + \alpha_j^+ - \alpha_j^- + \eta = 0 ~ ~ .
\end{equation}
We now need to examine three cases, depending on the relation between $p_j$
and $q_j$. First, when $|p_j-q_j|<1/\nu$, neither of the inequality
constraints is binding. Therefore, the complementary slackness conditions
imply that at the saddle point $\alpha_j^+(p_j-q_j-1/\nu)=0$ and
$\alpha_j^-(p_j-q_j+1/\nu)=0$. We thus must have $\alpha_j^+=\alpha_j^-=0$
at the optimum. In this case, $\partial_j\phi(\v{p})=\eta$ which is
case~(\ref{c5}) in the theorem statement. Next, if $p_j-q_j=1/\eta$, then
$\alpha_j^-=0$ while $\alpha_j^+\geq 0$, using again the complementary
slackness conditions for optimality. We therefore get that
$\partial_j\phi(\v{p})=\eta-\alpha_j^+\leq\eta$ which is case~(\ref{c6})
of the problem statement. The case $p_j-q_j=-1/\nu$ is derived analogously.
Finally, when $\v{m}\neq\v{1}$ we modify the simplex
constraint accordingly and replace $\alpha_j^{\pm}$ with $m_j\alpha_j^{\pm}$.
\end{proof}

The above lemma provides us with a simple certificate for the optimality of
a vector $\v{p}$, given $I_-,I_0,I_+$ and $\nu$. As we will see the partition
tends to remain intact as $\nu$ varies. This makes it possible to track
the solution path $\v{p}(\nu)$ as $\nu$ gradually increases. This approach is
known as the Homotopy method~\citep{OsbornePrTu00b}.

\subsection{Relaxation Path for Additively Separable Convex Objectives}
\label{separable_homotopy:sec}
We now arrive at the notion of a relaxation path. It is convenient to keep
the discussion general before narrowing down to our main subject, the relaxed
maximal entropy problem.  We restrict the objective $\phi(\cdot)$ from
\eqref{gen:eqn} to the case of an additively separable function,
\B \label{separable_obj:eqn}
\phi(\v{p}) = \sum_{j=1}^n m_j \, \phi_j(p_j) ~ ~ .
\E
We further constrain each function
$\phi_j:\reals\rightarrow\reals$ and assume that it is a strictly convex
and continuously differentiable over $(0,1)$. 

\paragraph{Existence of the relaxation path.}
Under these conditions,
we now show that there exists a function $\nu \mapsto \eta(\nu)$ 
defined over an interval $[0,\nu_\infty]$, given implicitly as the unique
solution to equation \eqref{hs2:eqn} below, which completely determines
the solution 
$\v{p}(\nu)$ to \eqref{gen:eqn}, through the relation
$$
   p_j(\nu) = q_j + \frac{1}{\nu}\,\capping\!\left(\nu\psi_j(\eta(\nu))-\nu q_j\right)
$$
where the functions $\theta(\cdot)$ and $\psi_j(\cdot)$ are given explicitly 
in the sequel. Thus, the term ``path'' for $\nu\mapsto\eta(\nu)$ is justified,
since knowledge of $\eta(\nu)$ allows us to determine $\v{p}(\nu)$ 
instantly.  We now unravel the structure of $\theta(\cdot)$ and
$\psi_j(\cdot)$. 

Given the assumption on the functions $\phi_j$, $\frac{{\rm d}\phi_j(p_j)}{{\rm d}p_j}$ is a monotonically increasing
continuous function. Let $0\le\psi_j\le 1$ denote its inverse function, i.e.,
$$ \psi_j\left(\frac{{\rm d}\phi_j(p_j)}{{\rm d}p_j}\right) = p_j
 ~ \mbox{ for } ~ 0<p_j<1 ~ ~ . $$
We can now invoke Lemma \ref{gen2:thm} and characterize the correct partition
$(I_-, I_0, I_+)$ in terms of the inverse function $\psi$ as follows:
\begin{eqnarray*}
  p_j \ge \psi_j(\eta) & j\in I_- \\
  p_j = \psi_j(\eta)   & j\in I_0 \\
  p_j \le \psi_j(\eta) & j\in I_+
\end{eqnarray*}
which implies that
\begin{equation} \label{ps:eqn}
p_j=\left\{\begin{array}{lll}q_j-1/\nu &\ \ \psi_j(\eta)\le q_j - 1/\nu
  \\ \psi_j(\eta) &\ \ |\psi_j(\eta) - q_j|<1/\nu \\
  q_j+1/\nu &\ \ \psi_j(\eta)\ge q_j + 1/\nu \end{array}\right. ~ ~ .
\end{equation}
\begin{figure}[t] \label{capping:fig}
%\begin{center}
%\scalebox{0.3} % Change this value to rescale the drawing.
%{
%\begin{pspicture}(0,-5.)(22.,5.)
%\psline[linewidth=0.2](0.,-3.)(8.,-3.)(14.,3.)(22.,3.)
%\psline[linewidth=0.1cm,linestyle=dashed,dash=0.16cm 0.16cm](0.0,0.0)(22.,0.0)
%\psline[linewidth=0.1cm,linestyle=dashed,dash=0.16cm 0.16cm](11.,5.0)(11.,-5.)
%\usefont{T1}{ppl}{m}{n}
%\rput(19.85,3.52){\Huge $\capping$}
%\usefont{T1}{ppl}{m}{n}
%\rput(14.5,3.5){\LARGE (1,1)}
%\usefont{T1}{ppl}{m}{n}
%\rput(8.5,-3.5){\LARGE (-1,-1)}
%\end{pspicture}
%}
%\end{center}
%\vspace{-0.75cm}
  \begin{center}
  \centerline{\psfig{figure=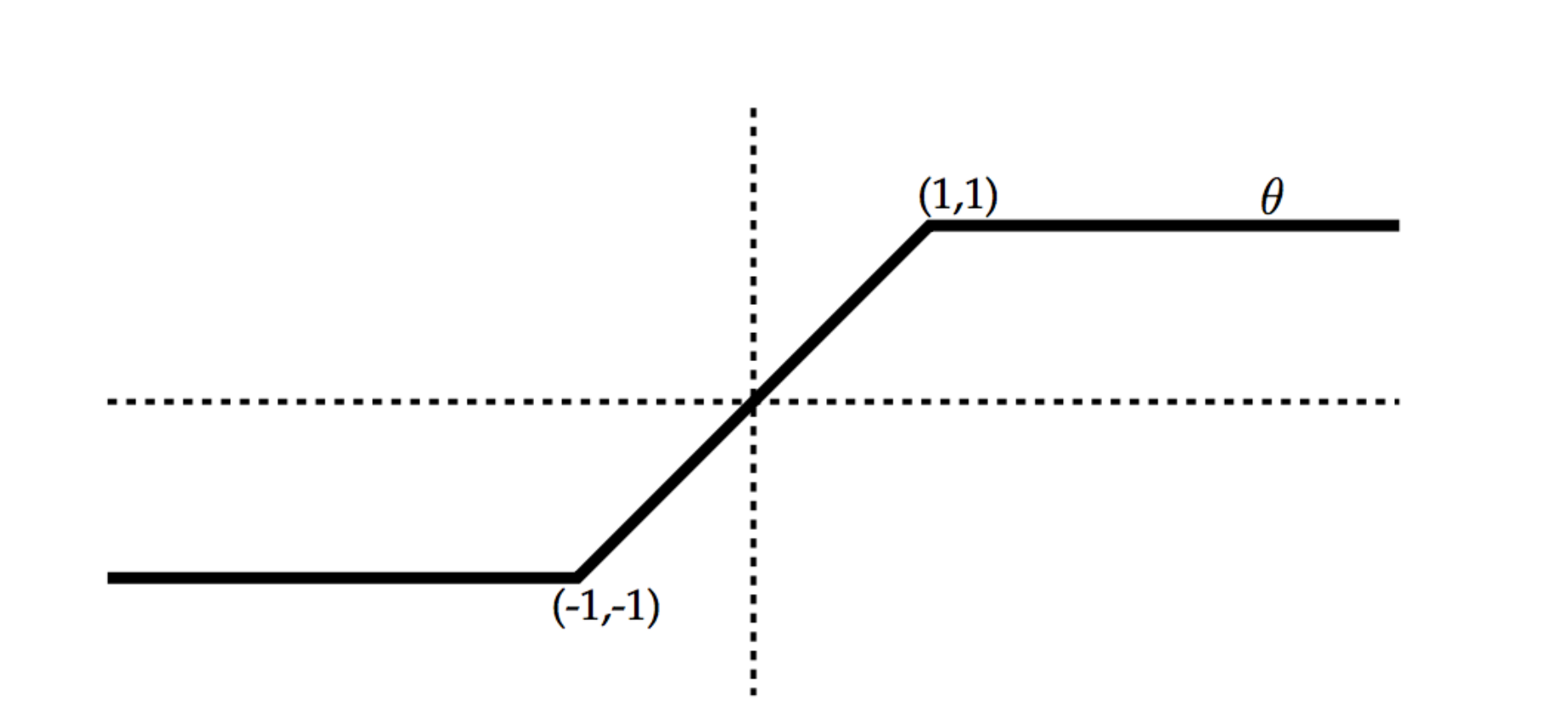,height=5cm}}
  \end{center}
\caption{The capping function $\capping$.}
\end{figure}
Denote by $\capping(\cdot)$ the capping function
$$ \capping(x)=\max\left\{-1, \min\left\{1, x\right\}\right\} ~ ~ . $$
An illustration of the capping function is given in Fig.~1.
Equipped with this definition we can rewrite (\ref{ps:eqn}) as follows
\B
  \forall j\in[n] : \qquad
   p_j = q_j + \frac{1}{\nu}\,\capping\!\left(\nu\psi_j(\eta)-\nu q_j\right)
   ~ ~ . \label{ps2:eqn}
\E
The solution is thus completely determined by a single unknown parameter
$\eta$, which depends on the problem variables and on $\nu$.
We now use the fact that that $\v{q}$ and $\v{p}$ are in $\simplex(\v{m})$,
sum \eqref{ps2:eqn} over $j$, and get
$$
\underbrace{\sum_{j=1}^n m_j p_j}_{=1} =
\underbrace{\sum_{j=1}^n m_j q_j}_{=1} +
	\frac{1}{\nu} \sum_{j=1}^n \,
	    m_j \capping\!\left(\nu\psi_j(\eta)- \nu q_j\right) ~  .
$$
The above equality leads to the following implicit equation for 
determining $\eta$
\B
G(\nu,\eta) \eqdef
\sum_{j=1}^{n} m_j \, \capping\!\left(\nu\psi_j(\eta)- \nu q_j\right)=0 ~ .
\label{hs2:eqn} \E

The correct value for $\eta$ is obtained by finding a zero of $G(\nu,\cdot)$.
The solution of $G(\nu,\eta)=0$ would give rise to the unique optimal
$\v{p}$ through (\ref{ps2:eqn}). Since the continuous function
$G(\nu,\eta)$ is monotonically non-decreasing from $G(\nu,-\infty)\le 0$ to
$G(\nu,\infty)\ge 0$, a solution (w.r.t. $\eta$) to the equation
$G(\nu,\eta)=0$ exists.
Moreover, the functions $\psi_j(\eta)$ are monotonically increasing, thus
the only setting in which we may obtain multiple solutions occurs when the
capping function is in its ``flat'' region. This can happen only if the
set $I_0$ is empty, thus $p_j = q_j \pm 1/\nu$. Let
\B
\nu_{\infty}\eqdef\inf\left\{ \nu\geq 0 \, \big| \, I_0(\nu) = \emptyset  \right\}
\,.
\label{nuinf:eqn}
\E
If the set $I_0$ is never empty for all finite values of $\nu$ we denote
$\nu_\infty=\infty$. Similarly, we denote by $\eta_\infty$ the solution
of $G(\nu_\infty,\eta)=0$ and define $\eta_\infty=\infty$ if
$\nu_\infty=\infty$. Note that for any $\nu\ge\nu_\infty$
$G(\nu,\eta_\infty)=0$, hence for $\nu\ge\nu_\infty$, $I_0(\nu)$ remains empty,
$I_-(\nu)=I_-(\nu_\infty)$, and $I_+(\nu)=I_+(\nu_\infty)$. We have thus
characterized the form of the solution $0\le\nu\le\nu_\infty$ through the
equation $G(\nu,\eta)=0$ which attains a unique zero at $\eta(\nu)$.
Increasing $\nu$ beyond $\nu_\infty$ does not change the form of the
solution (in terms of the partition into $I_+,I_-,I_0$) hence we can confine
the description of the relaxation path for $\nu$ to the interval
$(0,\nu_\infty]$. In summary, the path $\eta(\nu)$ exists for any separable
objective.

\section{Geometry of the Relaxation Path} \label{munu:sec}
The previous section provided an abstract characterization of the relaxation
path through the equation $G(\nu,\eta)=0$. While the path can in principle
be recovered for any individual value of $\nu$ by solving the equation
$G(\nu,\eta)=0$, we are interested in a computational feasible method for
finding it entirely. This task may not be possible for general additively
separable $\phi$.  Our main setting, the relaxed maximum entropy problem, is
an example where the relaxation path $\nu\mapsto \eta(\nu)$ admits a simple
geometric description.  If Sec. \ref{extensions:sec} we discuss an additional
case where the relaxation path admits a geometric description and a
corresponding tracking scheme.

In the relaxed maximum entropy, the objective $\phi(\v{p})$
is the relative entropy between the distribution
$\v{p}\in\simplex(\v{m})$ and a {\em known} distribution
$\v{u}\in\simplex(\v{m})$,
$$
\phi(\v{p}) = \sum_{j=1}^n m_j \phi_j(p_j) =
\sum_{j=1}^n m_j\,p_j\log\!\left(\frac{p_j}{u_j}\right) ~ ~.
$$
The optimization problem is then 
\begin{eqnarray} \label{relmaxent:eqn}
	& \displaystyle \min_{\v{p}\,\in\simplex(\v{m})} &
    \sum_{j=1}^n m_j \, p_j\log\!\left(\frac{p_j}{u_j}\right) \\
		& \hspace{1cm} \mbox{ s.t. } & \|\v{p}-\v{q}\|_\infty \leq
                1/\nu \nonumber ~.
\end{eqnarray}
We refer to $\v{u}$ as the {\em prior} distribution and to
$\v{q}$ as the {\em observed} distribution. (The term ``prior distribution''
 is not used here in any Bayesian context.) It is assumed that $u_j>0$ for
 $j=1,\ldots ,n$. It is useful to examine the dual problem of
 \eqref{relmaxent:eqn}, which can be shown to be
\B \label{maxent_dual:eqn}
-\min_{\v{\alpha}} \; \left\{
	\log\left(
	\sum_{j=1}^n m_j \, u_j \, {\rm e}^{\alpha_j}
	\right) \; - \;
  \sum_{j=1}^n m_j \, q_j \, \alpha_j \; + \;
	\frac{1}{\nu} \sum_{j=1}^n m_j\,|{\alpha}_j| \right\}
	~ ~ .
\E
Let $Z$ denote the sum
	$\sum_{j=1}^n m_j \, u_j \, {\rm e}^{\alpha_j}$.
Given a solution for the dual problem, the primal solution $\v{p}$ can be
reconstructed from $\v{\alpha}$ as follows
\B \label{psol:eqn}
p_j = \frac{u_j\,e^{\alpha_j}}{Z} ~ ~ .
\E
Moreover, we can rewrite the dual objective in a mixed form using
$\v{p}(\v{\alpha})$ as
$$
-\sum_{j=1}^n m_j q_j\log\left(p_j(\v{\alpha})\right) +
\frac{1}{\nu}\,\sum_{j=1}^n m_j |\alpha_j| ~ .
$$
When $\v{m}=\v{1} $, the dual form amounts to finding an exponential tilt of the
multinomial distribution $\v{u}$, with an $\ell_1$ penalty on the
exponential tilt coefficients. Adding a constant term, the dual objective
can be written in a mixed form as 
$$
D_{\text{KL}}(\v{q}\,||\,\v{p}(\v{\alpha})) + \frac{1}{\nu}\|\v{\alpha}\|_1 \,.
$$  
As the $\ell_1$ penalty tends to promote sparse solutions, the primal
problem can be interpreted as the task of finding a {\em sparse exponential
tilt}, namely, an exponential tilt $\v{p}$ of
the prior distribution $\v{u}$, in which $p_i\propto u_i$ for most $1\leq i\leq
n$, and which is close (in $D_{KL}$) to the observed
distribution $\v{q}$.

Building on these insights, we now turn to a geometric description of the
relaxation path in this case.
Our first step is to re-parameterize the problem by introducing a parameter
$\mu=\nu\,{\rm e}^{\eta+1}$.
In terms of $\mu$, equations \eqref{ps2:eqn} and
\eqref{hs2:eqn} amount to
\begin{eqnarray}
p_j & = & q_j + \frac{1}{\nu}\,\capping\!\left(\mu u_j-\nu q_j\right)
 ~ ~ , \label{psolcap:eqn}  \\
G(\nu,\mu) & = & \sum_{j=1}^{n} m_j\,\capping\!\left(\mu u_j-
 \nu q_j\right)=0 ~ ~ . \label{hs3:eqn}
\end{eqnarray}
Before proceeding, we would like to point to an aesthetic symmetry.
Note that $(\mu,\v{u})$ are interchangeable with $(\nu,\v{q})$. We can thus
swap the roles of the prior distribution with the observed distribution
and obtain an analogous characterization.

In order to explore the dependency of $\mu$ on $\nu$ let us introduce the
following sums
\begin{equation} \label{mnsets:eqn}
\s{M} = \sum_{j\in I_+}m_j - \sum_{j\in I_-}m_j ~ ~ , ~ ~
\s{U} = \sum_{j\in I_0} m_j u_j ~ ~ , ~ ~
\s{Q} = \sum_{j\in I_0} m_j q_j ~ ~ .
\end{equation}
Fixing $\nu$ while using (\ref{mnsets:eqn}), we can rewrite (\ref{hs3:eqn})
as 
\B \label{mnlin:eqn} \mu\,\s{U} - \nu \s{Q} + \s{M} = 0 ~ ~ .\E
Clearly, so long as the partition of $[n]$ into the sets $I_{+},I_{-},I_{0}$
is intact, there is a simple linear relation between $\mu$ and $\nu$.
The number of possible subsets $I_-,I_0,I_+$ is finite. Thus, the range
$0<\nu<\infty$ decomposes into a finite number of intervals each of which
corresponds to a fixed partition of $[n]$ into $I_{+},I_{-},I_{0}$.  Therefore,
in each
interval where $I_0$ is not empty, $\mu$ is a linear function of $\nu$.
Finally, recall that $I_0$ is not empty for $\nu < \nu_\infty$, where
$\nu_\infty$ is given by \eqref{nuinf:eqn}, and empty for $ \nu \geq
\nu_\infty$.

To recap our derivation, the following lemma characterizes of the solution of
$G(\nu,\cdot)$. We denote the relaxation path for $\mu$ with respect
to $\nu$ by $\mu(\nu)$.
\begin{lemma} \label{munu:lemma}
  For $0\le\nu\le\nu_\infty$, the value of $\mu$ as defined by
  (\ref{hs3:eqn}) is unique. Further, the function $\mu(\nu)$ is
  a piecewise linear continuous function in $\nu$. Increasing $\nu$
  beyond $\nu_\infty$ does not change $\v{p}$.
\end{lemma}
This establishes the fact that $\mu(\nu)$ is a piecewise linear function.
The lingering question is how many linear sub-intervals the function can attain.
To study this property, we take a geometric view of the plane defined by
$(\nu,\mu)$. Our combinatorial characterization of the number of sub-intervals
makes use of the following definitions of lines in $\reals^2$,
\begin{eqnarray}
  \lll_{+j} & = & \{(\nu,\mu)\ |\ u_j\mu-q_j\nu= +1\} \label{l-plus:eq} \\
\lll_{-j} & = & \{(\nu,\mu)\ |\ u_j\mu-q_j\nu= -1\} \label{l-minus:eq} \\
\lll_{0}  & = & \{(\nu,\mu)\ |\ \mu\,\s{U} - \nu\s{Q} + \s{M} = 0\}
\label{l-zero:eq} ~ ~ ,
\end{eqnarray}
where $-\infty<\nu<\infty$ and $j\in[n]$. The next theorem gives an upper
bound on the number of linear segments the function $\mu(\cdot)$ may attain.
While the bound is quadratic in the dimension, for both artificial data and
real datasets the bound is too pessimistic, as we demonstrate in the sequel.

\begin{figure}[t]
  \begin{center}
  \centerline{\psfig{figure=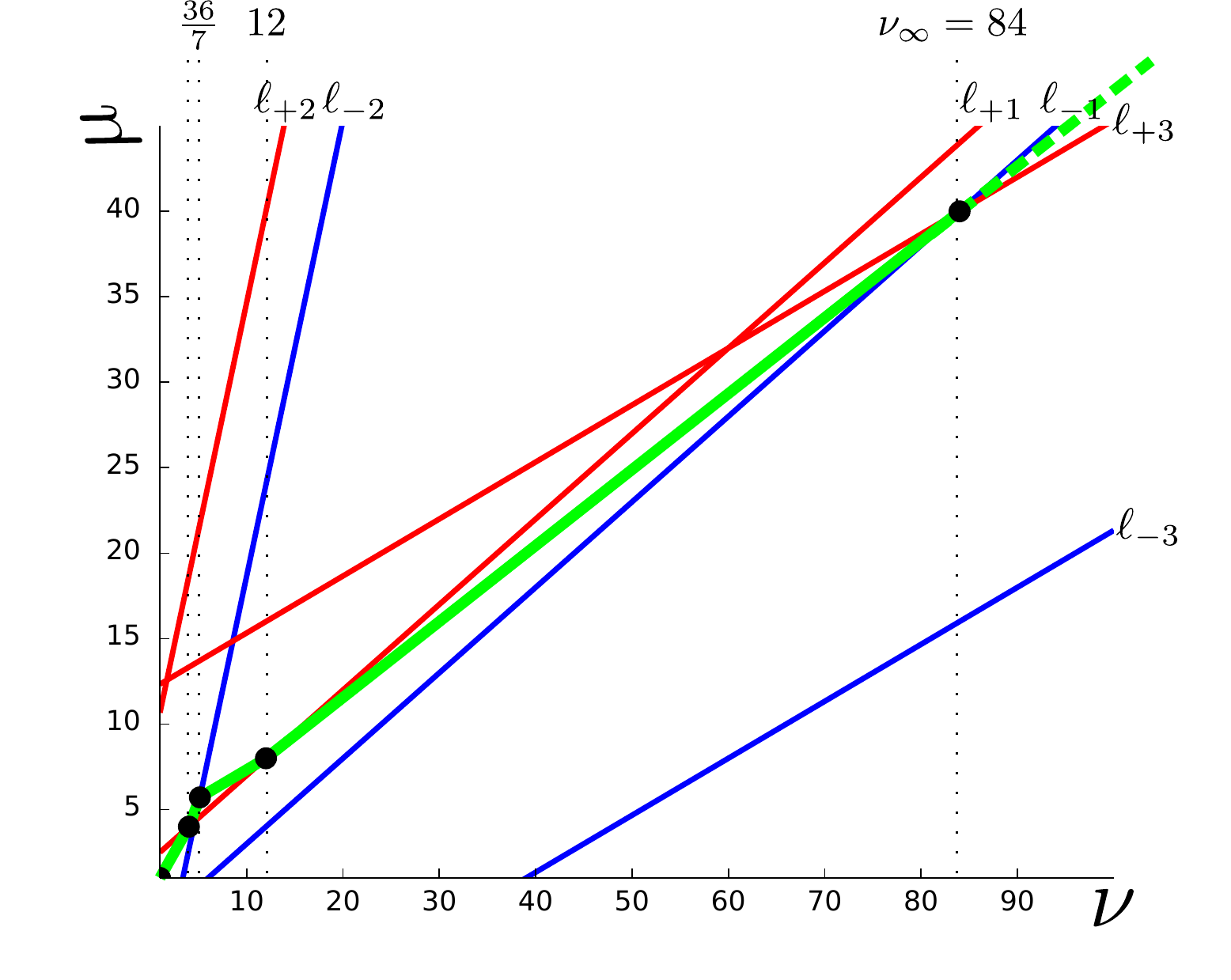,height=6cm}}
  \end{center}
  \vspace{-1cm}
  \caption{An illustration of the function $\mu(\nu)$ for a synthetic
  $3$ dimensional example.}
  \label{munu:fig}
\end{figure}

\begin{theorem} \label{quadseg:thm}
The piecewise linear function $\mu(\nu)$ consists of at most $n^2$
linear segments for $\nu\in\reals_+$.
\end{theorem}
\begin{proof}
Since we showed that that $\mu(\nu)$ is a piecewise linear function, it remains
to show that it has at most $n^2$ linear segments. Consider the two dimensional
function $G(\nu,\mu)$ from (\ref{hs3:eqn}). The ($\nu,\mu)$ plane is divided
by the $2n$ straight lines
$\lll_1,\lll_2,\ldots,\lll_n,\lll_{-1},\lll_{-2},\ldots,\lll_{-n}$
into at most $2n^2+1$ polygons.
The latter property is proved by induction. It clearly holds for $n=0$.
Assume that it holds for $n-1$. Line $\lll_n$ intersects the previous $2n-2$
lines in at most $2n-2$ points, thus splitting at most $2n-1$ polygons into two
separate polygonal parts. Line $\lll_{-n}$ is parallel to $\lll_n$, again adding
at most $2n-1$ polygons. Together we get at most
$2(n-1)^2+1+2(2n-1)=2n^2+1$ polygons, as required per induction.
Recall that $\mu(\nu)$ is linear inside each polygon.
The two extreme polygons where
$G(\nu,\mu) = \pm n$ clearly disallow $G(\nu,\mu)=0$, hence
$\mu(\nu)$ can have at most $2n^2-1$ segments for $-\infty<\nu<\infty$.
Lastly, we use the symmetry $G(-\nu,-\mu)=-G(\nu,\mu)$ which implies that for
$\nu\in\reals_+$ we have at most $n^2$ segments, as required.
\end{proof}
When the prior $\v{u}$ is uniform, i.e. $u_j=1/\sum_j m_j$ for all $j\in[n]$,
the number of segments is at most $n+1$. We defer the analysis of the uniform
case to a later section as the proof stems from the algorithms we describe
in the sequel.

To conclude this section, we characterize the path $\mu(\nu)$ in
the following toy example. Let
$\v{u}=({1}/{2}, {1}/{8}, {1}/{12})$,
$\v{q}=({1}/{4}, {1}/{3}, {1}/{36})$, and
$\v{m}=(1, 2, 3)$. Note that $\v{u} \cdot \v{m} = \v{q} \cdot \v{m} = 1$ as
required from our definition of distributions with multiplicity. The complete
characterization of $\mu(\nu)$ for $0<\nu<\infty$ is as follows,
\begin{center}
\begin{tabular}{|c|c|c|c|c|} \hline
  region of $\nu$ & $I_-$ & $I_0$ & $I_+$ & $\mu(\nu)$ \\ \hline
  $0<\nu<4$ & $\{\}$ & $\{1,2,3\}$ & $\{\}$ &
    $\mu = \nu$ \\
  $4\le\nu<\frac{36}{7}$ & $\{\}$ & $\{2,3\}$ & $\{1\}$ &
    $\mu = 4 + \frac{3}{2}(\nu-4)$ \\
  $\frac{36}{7}\le\nu<12$ & $\{2\}$ & $\{3\}$ & $\{1\}$ &
    $\mu = \frac{40}{7} + \frac{1}{3}(\nu -\frac{36}{7})$ \\
  $12\le\nu<84$ & $\{2\}$ & $\{1,3\}$ & $\{\}$ &
  $\mu = 8 + \frac{4}{9}(\nu - 12)$ \\
  $\nu_\infty=84\le\nu$ & $\{1,2\}$ & $\{\}$ & $\{3\}$ &
    $\mu = 8 + \frac{4}{9}(\nu - 12)$ \\
  \hline
\end{tabular}
\end{center}

\medskip

The above table implies that $\nu_\infty=84$ and $\mu_\infty=40$. Thus,
the partition $I_-=\{1,2\}$, $I_+=\{3\}$, and $I_0$ is empty remains intact
for any $\nu\ge 84$. Figure \ref{munu:fig} shows the constraint lines
$\ell_{\pm 1},\ell_{\pm 2}, \ell_{\pm 3}$, the path segments and the path
itself. Interestingly, note that the cardinalities
of $I_0$ and $I_+$ are not monotone. Indeed, the first coordinate enters
$I_+$ from $I_0$ and then returns to $I_0$, finally ending at $I_-$ for
$\nu\ge\nu_\infty$. This kind of non-monotone behavior is the reason why
$O(n^2)$ linear segments are necessary to describe $\mu(\nu)$ in the worst
case.

\section{Path Tracking Algorithms} \label{local_homotopy:sec}
In this section we build on the geometric description above and discuss
algorithms for tracking the maximum entropy relaxation path.  The algorithms
are based on a local search for the next intersection of the line $\ell_0$ with
one of the lines $\ell_{\pm j}$.  These algorithms are simple to implement and
efficient in practical settings.  In Appendix~\ref{global_homotopy:sec} we
outline a more complicated algorithm with slightly better worst case
performance, which maintains global information of the homotopy.  Discussion of
the global tracking algorithm is deferred to the appendix as it is not
straightforward to implement.

\subsection{Local Homotopy Tracking}
Since we showed that the
optimal solution $\v{p}$ can be straightforwardly obtained from the variable
$\mu$, it suffices to devise an algorithm that efficiently tracks the function
$\mu(\nu)$ as we traverse the plane $(\nu,\mu)$ from $\nu=0$ through the last
change point which we denoted as $(\nu_\infty,\mu_\infty)$. In this section we
give an algorithm that traces $\mu(\nu)$ by tracking the changes in $\mu(\nu)$
through a local search process. Concretely, we start by computing the initial
slope of $\lll_0$ at $\nu=0$. We then find the closest intersection with a line
$\lll_j$ (for $1 \le |j| \le n$) and calculate the new slope of $\lll_0$ as the
intersection with the line induces a new partition into the sets $I_+,I_-,I_0$.
We continue this process until we reach the point $(\nu_\infty,\mu_\infty)$
beyond which the partition into $(I_\pm,I_0)$ does not change.

More formally, the local tracking algorithm follows the piecewise linear
function $\mu(\nu)$, segment by segment. Each segment corresponds to a subset
of the line $\lll_0$ for a {\em given} triplet $(\s{M},\s{U},\s{Q})$.  It is
simple to show that $\mu(0)=0$, hence we start with $(\nu,\mu)=(0,0)$.  Given
the pair $(\nu,\mu)$ the partition into the sets $I_\pm$ and $I_0$ is
straightforward as we can rewrite (\ref{Isets:eqn}) as,
\begin{eqnarray*}
  I_+ & = & \{1\le j\le n\ |\ \mu\,u_j - \nu\,q_j\geq 1\} \\
  I_0 & = & \{1\le j\le n\ |\ |\mu\,u_j - \nu\,q_j| < 1\} \\
  I_- & = & \{1\le j\le n\ |\ \mu\,u_j - \nu\,q_j \leq 1\} ~ ~ .
\end{eqnarray*}
This form of index partitioning implies that given $(\nu,\mu)$ we can
calculate $\s{M},\s{U},\s{Q}$ directly as follows,
\begin{equation} \label{muq_direct:eqn}
\s{M}\; = \!\!\! \sum_{\mu u_j-\nu q_j\ge 1} \hspace{-0.45cm} m_j \; -
\!\!\!
\sum_{\mu u_j-\nu q_j\le -1} \hspace{-0.5cm} m_j \hspace{1cm}
\s{U}\; = \!\!\! \sum_{|\mu u_j-\nu q_j|<1} \hspace{-0.5cm} m_j\, u_j
\hspace{1cm}
\s{Q}\; = \!\!\! \sum_{|\mu u_j-\nu q_j|<1} \hspace{-0.5cm} m_j\,q_j  ~ ~ .
\end{equation}
From the triplet $(\s{M},\s{U},\s{Q})$ the initial characterization of
$\lll_0$ is readily available as we can write,
\begin{equation} \label{muofnu:eqn}
\mu = \frac{\nu\, \s{Q} - \s{M}}{\s{U}}
\; = \;
\frac{\s{Q}}{\s{U}} \, \nu - \frac{\s{M}}{\s{U}} 
~ ~ .
\end{equation}
In words, the line $\ell_0$ has a slope of $\s{Q}/\s{U}$ and an intercept
of $-\s{M}/\s{U}$. Initially the set $I_0=[n]$, $\s{M}=0$, and therefore
$\s{Q}=\s{U}=1$, which implies that the initial slope of $\ell_0$ is $1$
and the intercept is $0$.
We now track the value of $\mu$ as $\nu$ increases (and the original relaxation
parameter $\delta$ decreases). The characterization of the line $\lll_0$
remains intact until $\lll_0$ hits one of the lines
$\lll_j$ for $1 \leq |j| \le n$. To find
the line intersecting $\lll_0$ we need to compute the potential intersection
points ($\mu_j,\nu_j)$ for
$\nu_{-n},\nu_{-n+1},\ldots,\nu_{-1},\nu_1,\nu_2,\cdots,\nu_{n}$ where
$(\nu_j, \mu_j) = \lll_0 \cap \lll_j$. This amounts to calculating the
potential intersection values,
\begin{equation} \label{intersect:eqn}
\nu_j =
\frac{\s{M} u_{|j|} + \s{U}\cdot{\rm sign}(j)}{\s{Q} u_{|j|} - \s{U} q_{|j|}}
\qquad;\qquad {\rm sign}(j)=\begin{cases}
  1 & j > 0 \\
  -1 & j <0
\end{cases}\,.
\end{equation}
The lines for which the denominator is zero correspond to an infeasible
intersection and can be discarded. The smallest value $\nu_j$ which is larger
than the current recorded value of $\nu$ (i.e. the last observed intersection
of $\lll_0$ with one the lines $\lll_j$) corresponds to the next line
intersecting $\lll_0$. From $\nu_j$ we compute $\mu_j$ using
(\ref{muofnu:eqn}). We now can construct the next segment of $\lll_0$, which
starts at $(\nu_j,\mu_j)$ by calculating a new value for the triplet
$(\s{M},\s{U},\s{Q})$ as prescribed by (\ref{muq_direct:eqn}). The homotopy
tracking process finishes once we cannot find any pair $(\nu_j,\mu_j)$ for
which $\nu_j$ is greater than the most recently traced found for $\mu$. That
is, the last intersection that was found corresponds to
$(\nu_\infty,\mu_\infty)$.

In the above description of the local tracking algorithm, the formation of the sets
$I_{\pm}$ and $I_0$ is tacit. Moreover, calculating the sums $\s{M},\s{Q}$,
and $\s{U}$ from scratch upon every newly found intersection of $\lll_0$ with
$\lll_j$ is not mandatory since each such intersection corresponds to moving a
{\em single} constraint $|p_j-q_j|\leq 1/\nu$ from $I_\pm$ to $I_0$ or vice
versa. By explicitly tracking the set $I_+$,$I_-$, and $I_0$ as they change, we
can update the sums $\s{M},\s{Q}$, and $\s{U}$ in a constant time upon each
newly encountered intersection. We therefore present an equivalent, yet more
efficient, procedure in which the sets $\s{M},\s{Q}$, and $\s{U}$ are updated
incrementally. Further, we use the latter property and the more
elaborate tracking scheme in the next section in which we analyze the case
where the prior distribution $\v{u}$ is uniform.
\begin{figure}[t] \label{updown:fig}
  \centerline{\epsfig{figure=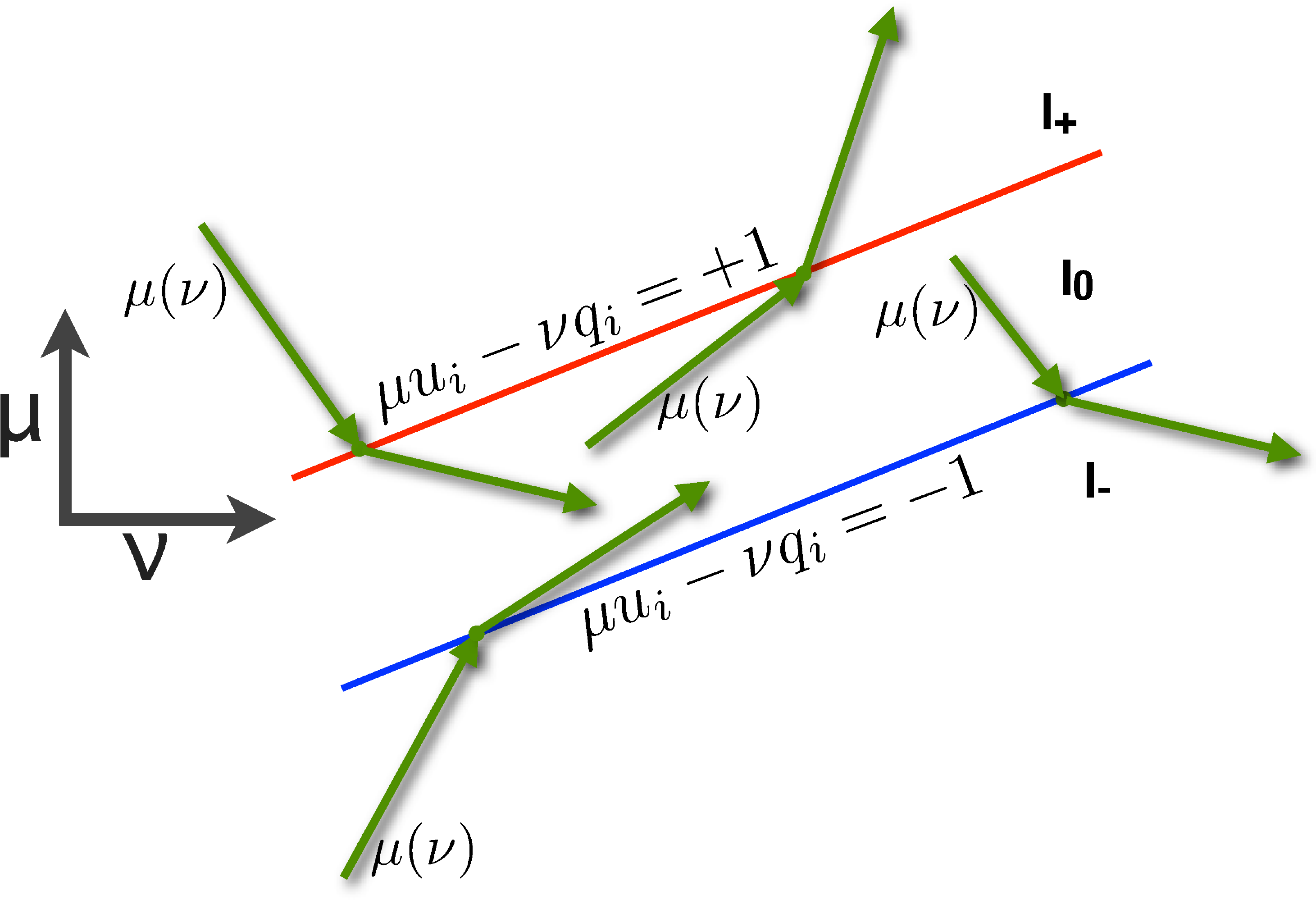,height=6cm}}
\caption{Illustration of the possible intersections between $\mu(\nu)$
and $\lll_j$ and the corresponding transition between the sets $I_\pm,I_0$.}
\end{figure}
As before, our goal is to track the piecewise linear function $\mu(\nu)$,
segment by segment where each segment forms a straight line
$\lll_0(\s{M},\s{U},\s{Q})$. In the alternative view, we track the sets
$I_-,I_0,I_+$ by defining an auxiliary variable per coordinate. Given the
current partition of $[n]$ into sets we denote,
\begin{equation}
  s_j = \left\{\hspace{-0.1cm}\begin{array}{rl}
  -1 & j\in I_-\\
   0 & j\in I_0\\
  +1 & j\in I_+
 \end{array} \right. ~ ~ .
\end{equation}
As in the tacit version, the explicit version starts with
$(\nu,\mu)=(0,0)$, $\s{M}=0$, and $\s{U}=\s{Q}=1$. Identically,
on each step we compute the $2n$ intersection values
$(\nu_j, \mu_j) = \lll_0 \cap \lll_j$ for $1\leq |j|\leq n$ and find the
nearest intersection. However, at this point the two procedures depart. Rather
than tacitly deferring the identification of the type of intersection to
the computation of the new sums, we explicitly characterize the form of
the change in the sets $I_\pm$ and $I_0$ due to the newly found intersection.
Recall that $\frac{\s{Q}}{\s{U}}$ is the left slope of $\mu(\nu)$ as
represented by the current line segment $\lll_0$. The slope of $\lll_j$
is $\frac{q_{|j|}}{u_{|j|}}$. Thus, when
$\frac{\s{Q}}{\s{U}}>\frac{q_{|j|}}{u_{|j|}}$ the $|j|$'th constraint is
moving ``up'' from $I_-$ to $I_0$ or from $I_0$ to $I_+$. When
$\frac{\s{Q}}{\s{U}}<\frac{q_{|j|}}{u_{|j|}}$ the $|j|$'th constraint is
moving ``down'' from $I_+$ to $I_0$ or from $I_0$ to $I_-$. See also
Fig.~4 for an illustration of the possible transitions
between the sets. For instance, the slope of $\mu(\nu)$ on the bottom left
part of the figure is larger than the slope the line it intersects. Since
this line defines the boundary between $I_-$ and $I_0$, we transition from
$I_-$ to $I_0$. All four possible transitions are depicted in the figure.
Thus, we need to consider only indices $j$ such that (\ref{intersect:eqn})
is defined and
\B (\s{Q}u_{|j|}-\s{U}q_{|j|}) \, s_{|j|} \, \le \, 0 \label{cond:eqn} ~ ~ . \E
Again, let $\nu_j$ be the smallest intersection value satisfying
(\ref{cond:eqn}). If there is no such value, we are done with the homotopy
tracking process. Otherwise, we can now update the sums $\s{M},\s{Q}$ and $\s{U}$
based on the single transition of the $j$'th element between the
characteristic sets. We also need to update $s_j$ itself. By isolating
the $j$'th term in (\ref{muq_direct:eqn}), the update of all the sums and
$s_j$ now takes the following incremental form,
\begin{eqnarray*}
	s_{|j|} & \leftarrow & s_{|j|} + {\rm sign}(\s{Q}u_{|j|}-\s{U}q_{|j|}) \\
\s{M} & \leftarrow & \s{M}+{\rm sign}(\s{Q}u_{|j|}-\s{U}q_{|j|})\,m_{|j|} \\
\s{U} & \leftarrow & \s{U}+{\rm sign}(j)\,
  {\rm sign}(\s{Q}u_{|j|}-\s{U}q_{|j|})\,m_{|j|}\,u_{|j|} \\
\s{Q} & \leftarrow & \s{Q}+{\rm sign}(j)\,
{\rm sign} (\s{Q}u_{|j|}-\s{U}q_{|j|})\,m_{|j|}\,q_{|j|}  ~ ~ .
\end{eqnarray*}
We are done with the tracking process when $I_0$ is empty, i.e. for all $j$
$s_j \neq 0$. The pseudo code of the entire process is provided in
Algorithm~\ref{alg:local_track}.

\begin{algorithm}[t]
	\caption{The local tracking algorithm for relaxed maximum entropy.}
  \label{alg:local_track}
\begin{algorithmic}[1]
	\STATE {\bf input:} Distributions $\v{q}$,$\v{u}$  ;  Multiplicity: $\v{m}$
	\STATE {\bf initialize:} $\s{Q} = \s{U} = 1$, $\s{M} = 0$ ,
	$L = \{(0,0)\}$, $\nu_{last} = 0$,  $\forall j\in[n]: s_j = 0$
	\WHILE{$\exists j \mbox{ s.t. } s_j = 0$}
		\STATE $\nu_{c} = \infty$
		\FORALL{$j\in \{-n,\dots,-1,1,\dots,n\}$}
			\IF{$(\s{Q}u_{|j|}-\s{U}q_{|j|}) s_j \le 0$}
				\STATE $\nu = \frac{\s{M} u_{|j|} + \s{U}{\rm sign}(j)}{\s{Q} u_{|j|} - \s{U} q_{|j|}}$
				\IF{$\nu <\nu_{c}$}
					\STATE $\nu_{c} \leftarrow \nu ; j_{c} \leftarrow j$
				\ENDIF
	\ENDIF
		\ENDFOR

				\IF{$\nu_{c}=\infty$}
					\STATE break
				\ENDIF
				\STATE $\mu_{c} = \frac{\nu_{c} \s{Q} - \s{M}}{\s{U}}$ ~ ; ~
				$L \leftarrow L \cup \{(\nu_{c},\mu_{c})\}$
				\STATE $ s_{|j_c|} \leftarrow s_{|j_c|} +
				{\rm sign}(\s{Q}u_{|j_c|}-\s{U}q_{|j_c|})$
				\STATE $\s{M} \leftarrow \s{M}+
				{\rm sign}(\s{Q}u_{|j_c|}-\s{U}q_{|j_c|})\,m_{|j_c|}$
				\STATE $\s{U} \leftarrow \s{U}+{\rm sign}(j_c)\,
				{\rm sign}(\s{Q}u_{|j_c|}-\s{U}q_{|j_c|})\,m_{|j_c|}\,u_{|j_c|}$
				\STATE $\s{Q} \leftarrow \s{Q}+{\rm sign}(j_c)\,
				{\rm sign} (\s{Q}u_{|j_c|}-\s{U}q_{|j_c|})\,m_{|j_c|}\,q_{|j_c|}$
				\ENDWHILE
		\STATE {\bf return} $L$
\end{algorithmic}
\end{algorithm}

\paragraph{Complexity.} The local tracking algorithm requires $O(n)$ memory
and $O(n\numchange)$ operations where $\numchange$ is the number of change
points in the function $\mu(\nu)$. When $\numchange$ is relatively small, this
algorithm is simple and efficient to implement. In Appendix
\ref{global_homotopy:sec} we give a more complicated algorithm, which employs
an auxiliary priority queues and requires fewer number of operations when
$\numchange > n\log(n)$. A illustration of the tracking result, $\mu(\nu)$,
along with the lines $\lll_{\pm j}$, that provide a geometrical description of
the problem, is given in Fig.~\ref{munu:fig}.

\subsection{Fast Homotopy Tracking for Sparse Observations}
\label{local-homotopy-sparse-obs:subsec}
In numerous practical settings, while the dimension $n$ may be very large, the
number of zero entries in $\v{q}$ can be substantial. We now discuss an
improvement to the local tracking algorithm that renders its feasible for very
large dimension $n$ so long as $\v{q}$ is sparse. Let us denote by $s$ the
support of $\v{q}$, $s:= |\left\{ 1\leq j\leq n \,|\, q_j \neq 0 \right\}|$,

Recall that the principle underlying the local tracking algorithm is that
every coordinate $j$ induces two lines in the $(\nu,\mu)$ plane, denoted
$\ell_{\pm j}$ as given by \eqref{l-plus:eq} and \eqref{l-minus:eq}. A
coordinate for which $q_j=0$ corresponds to a {\em horizontal} line $\ell_{\pm
j}$ described by the equation by $u_j\mu = \pm 1$. Since the path $\mu(\nu)$ is
non-decreasing with $\mu(0)=0$, each horizontal line $\ell_{-j}$ resides
outside the positive quadrant in $(\nu,\mu)$ plane and is never intersected.
Further, each of the lines $\ell_{+j}$ is intersected exactly once. Note that
$u_j > u_k$ implies that the line $\ell_{+j}$ is intersected at
$\mu=\frac{1}{u_j}$, before the line $\ell_{+k}$, whose intersection is at
$\mu=\frac{1}{u_k}$. Therefore, we sort the values $\{u_j\, | \, q_j = 0\}$ in
decreasing order, as a preliminary step. Then, the search for the next
intersection in the local tracking algorithm can be confined to scanning
$2s+1$ lines only. Namely, the $2s$ lines corresponding to nonzero values of
$\v{q}$ and the next horizontal line to be intersected from the zero set of
$\v{q}$.

Figure \ref{sparse-q:fig} provides an illustration of the constraint lines and
the path in the $(\nu,\mu)$ plane when $\v{q}$ is sparse. In this
$8$-dimensional toy example
for sparse observations,
$\v{u} = (0.0372,0.0445,0.0403,0.0144,0.0268,0.0088,0.0389,0.0390)$ is a vector
whose entries were sampled from a uniform distribution on $[0,1]$ and
normalized, $\v{q}=(0,\frac{5}{8},0,0,0,0,\frac{3}{8},0)$, and 
 $\v{m}=(0,10,0,0,0,0,10,0)$. Observe that each coordinate where $q_j=0$
 donates a single horizontal line, and that these lines are intersected
 according to the order of the corresponding values $u_j$.

    This fast version of the local tracking algorithm requires $O(n)$ memory
and $O(n\log n + s\kappa)$ operations where $\kappa$ is the number of path
change points. From Theorem \ref{quadseg:thm} we have in this case
$\kappa\leq s^2 + n$ so that the worst case time complexity of the fast
algorithm is $O(n\log n+ sn+ s^3)$. Therefore, in the practical case where
the sparsity $s$ of the observed distribution
is logarithmic in the dimension the total complexity even in the worst case
becomes $O(n \log n)$, which makes the algorithm practical for very high
dimensional problems.

\begin{figure}[t]
  \begin{center}
    \centerline{\psfig{figure=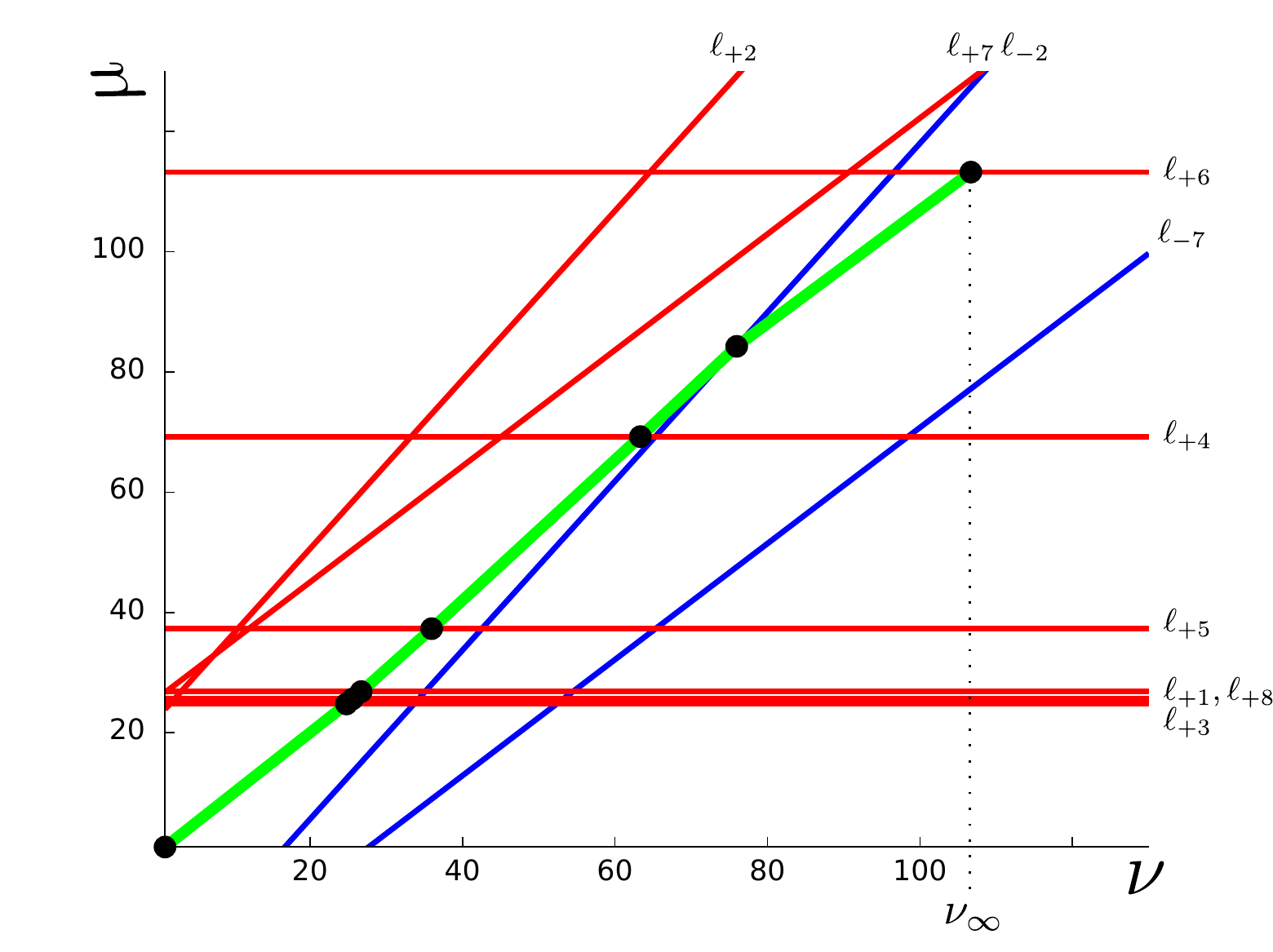,width=9cm}}
  \end{center}
  \vspace{-1cm}
  \caption{The relaxation path for a toy 8-dimensional example with a 2-sparse
  observations vector -- intersects mostly horizontal constraint lines}
  \label{sparse-q:fig}
\end{figure}

\subsection{Fast Homotopy Tracking for Uniform Prior} \label{uniform:sec}
We chose to denote the prior distribution by $\v{u}$ to underscore the fact
that in the case of no prior knowledge $\v{u}$ is the {\em uniform}
distribution. When the prior is uniform, then for all $j\in[n]$ the value
of $u_j$ is the same and is equal to
$$ \uu \eqdef {\left(\sum_{j=1}^{n}m_j\right)}^{-1} ~ ~ .  $$
Moreover, for a uniform prior objective function amounts to the negative
entropy. By reversing the sign of the objective, we obtain the classical
maximum entropy problem. Similarly to the case of sparse observed vector, the
case of a uniform prior distribution simplifies the geometry and consequently
significantly lowers the complexity of the tracking algorithm.

Let us consider a point $(\nu,\mu)$ on the boundary between $I_0$
and $I_+$, namely, there exists a line $\lll_{+i}$ such that,
$$ \mu u_i - \nu q_i = \mu \uu - \nu q_i = 1 ~~ . $$
By definition, for any $j\in I_0$ we have
$$ \mu u_j - \nu q_j =  \mu \uu  - \nu q_j  < 1 = \mu \uu - \nu q_i  ~~ . $$
Thus, $q_i<q_j$ when $i\in I_+$ and  for all $j\in I_0$. The inequality
implies that
\begin{equation} \label{uniforder:eqn}
m_j \, \uu \, q_j \;  > \;  m_j \, \uu \, q_i ~~ .
\end{equation}
Summing~(\ref{uniforder:eqn}) over $j \in I_0$ we get that
$$ \s{Q}\,\uu \;=\; \sum_{j\in I_{0}}  m_{j}\,q_{j}\,\uu  \; > \;
     \sum_{j\in I_{0}} m_{j}\,\uu\,q_{i} \; =\; \s{U} q_i ~ ~ , $$
hence,
$$\frac{q_i}{u_i} = \frac{q_i}{\uu}<\frac{\s{Q}}{\s{U}}$$
and we must be moving ``up'' from $I_0$ to $I_+$ when the line $\lll_{0}$ hits
a line $\lll_{i}$. Similarly we must be moving ``down'' from when $\lll_{0}$
intersects a line on the boundary between $I_0$ and $I_-$. We summarize these
properties in the following theorem. Here, we say that the set $I(\nu)$ is {\em
monotonically non-decreasing} if $i\in I(\nu)$ implies $i\in I(\nu_1)$ whenever
$\nu\leq\nu_1$. 
\begin{theorem}
When the prior distribution $\v{u}$ is uniform, $I_-(\nu)$ and $I_+(\nu)$ are
monotonically nondecreasing and $I_0(\nu)$ is monotonically non-increasing  in
$\nu>0$ . Further, the piecewise linear function $\mu(\nu)$ consists of at most
$n+1$ line segments.
\end{theorem}

The local tracking algorithm when the prior is uniform is particularly
simple and efficient. Intuitively, there is a single condition which controls
the order in which indices enter $I_{\pm}$ from $I_{0}$, which is simply
how ``far'' each $q_{j}$ is from $\uu$, the single prior value. Therefore, the
algorithm starts by sorting $\v{q}$.
Let $q_{\pi_{1}} > q_{\pi_{2}} > \cdots > q_{\pi_{n}}$
denote the sorted vector.  Instead of maintaining a vector of
set-indicators $\v{s}$, we merely maintain two indices which denote as $j_-$
and $j_+$. These indices designate the size of $I_{-}$ and $I_{+}$ that were
constructed thus far. Due to the monotonicity property of the sets $I_\pm$,
as $\nu$ grows, the two sets can be written as,
$$
I_- = \{\pi_{j} \, | \, 1\le j<j_-\}
~~ \mbox{ and } ~~
I_+ = \{\pi_{j} \, |\,  j_+<j\le n\} ~ ~ .
$$
The local tracking algorithm starts as before with $\nu=0$, $\s{M}=0$,
$\s{U}=\s{Q}=1$.  We also set $j_-=1$ and $j_+=n$ which by definition imply
that $I_{+}$ and $I_{-}$ are empty, and $I_{0}=[n]$. On each iteration we
need to compare only two values which we compactly denote as,
$$
\nu_{\pm} =
  \frac{\s{M}\,\uu\;\pm\;\s{U}}{\s{Q}\,\uu\;-\;\s{U}\,q_{\pi_{j_{\pm}}}} ~~ .
$$
When $\nu_- \le \nu_+$ we encounter a transition from $I_{0}$ to
$I_{-}$ and as we encroach $I_{-}$ we perform the update
\begin{eqnarray*}
\nu & \leftarrow & \nu_- \\
\s{M} & \leftarrow & \s{M} \, - \, m_{\pi_{j_{-}}} \\
\s{U} & \leftarrow & \s{U} \, - \, m_{\pi_{j_-}} \uu \\
\s{Q} & \leftarrow & \s{Q} \, - \, m_{\pi_{j_{-}}} q_{\pi_{j_{-}}} \\
j_-  & \leftarrow & j_- \,+\, 1  ~ ~ .
\end{eqnarray*}
Similarly when $\nu_- > \nu_+$ we perform the update
\begin{eqnarray*}
\nu & \leftarrow & \nu_+ \\
\s{M} & \leftarrow & \s{M} \, + \, m_{\pi_{j_{+}}} \\
\s{U} & \leftarrow & \s{U} \, - \, m_{\pi_{j_+}} \uu \\
\s{Q} & \leftarrow & \s{Q} \, - \, m_{\pi_{j_{+}}} q_{\pi_{j_{+}}} \\
j_+  & \leftarrow & j_+ \,-\, 1  ~ ~ .
\end{eqnarray*}
The local tracking algorithm stops when $j_{-}>j_{+}$ as we exhausted the
transitions out of the set $I_{0}$, which becomes empty. We have thus shown
that the local tracking algorithm for a uniform prior requires $O(n)$
memory and $O(n\log(n))$ operations.

\section{Model Selection Along the Relaxation Path} \label{cv:sec}
In this section we study the following attractive property of the
relaxed maximum entropy problem.
 Suppose that we evaluate models by their
likelihood on a held-out validation set. As we will see, there is a unique admissible
model for each possible model size. Equivalently, in terms of our geometric
description of the relaxation path, there is only one admissible
relaxation parameter per path interval. Moreover, this
discrete family of possible relaxation parameters can be recovered efficiently
to arbitrary precision, without performing a grid search over the $\nu$
variable.

\paragraph{Setup.} Once the maximum entropy problem \eqref{relmaxent:eqn} is
solved, the distribution $\v{p}$ is efficiently characterized for each
possible relaxation value $\nu$. Assume that we have calculated the relaxation
path for given vectors $\v{u}$ and $\v{q}$. To make the dependence on $\nu$
explicit, we write $\v{p}(\nu)$ for the primal solution corresponding to the
relaxation parameter $\nu$. Having solved for the entire relaxation path, we
can evaluate the map $\nu\mapsto \v{p}(\nu)$ at any $\nu\geq 0$. This allows
us the luxury of having all possible values of $\nu$  to choose from. 

Assume that we have available validation data in the form of a vector of
counts, $\v{r}\in\mathbb{N}^n$. (We can normalize $\v{r}$ to be in the
probability simplex without changing the value of the optimal
solution.) To choose $\nu$, we minimize w.r.t. $\nu$ the negative empirical
log-likelihood for the validation data, $-\sum_i r_i \log p_i(\nu)$.  As we
will shortly see, this minimization is easiest to handle when the optimization
parameter is the {\em inverse} of the relaxation parameter $\nu$ we have
used so far, denoted by $\lambda = \frac{1}{\nu}$. We thus need to solve the
problem
\begin{eqnarray*}
  {\lambda}^\star = \arg\min_{0\leq\lambda\leq 1} L_r(\lambda) \eqdef
		-\sum_j r_j \log p_j(\lambda) \, .
\end{eqnarray*}

\paragraph{Efficient minimization of the validation likelihood on each path segment.}
Using $\lambda$, it is convenient to express the primal solution
\eqref{hs2:eqn} as follows,
\begin{eqnarray}
  p_j(\lambda) = 
  \begin{cases}
    q_j + \lambda \, & j\in I_+ \\
		{u_j}{\mu}\,{\lambda} \,  & j \in I_0 \\
    q_j - \lambda \, & j \in I_- 
  \end{cases} \,,
  \label{solution-thresh-form:eqn}
\end{eqnarray}
where $I_{\pm}$ and $I_0$ are given by \eqref{Isets:eqn} and depend on 
$\nu$ (or, equivalently, on $\lambda$).
We can thus write the objective $L_{\v{r}}(\lambda)$ as
\begin{eqnarray}
  -\sum_{j\in I_+}r_j \log(q_j + \lambda) - \sum_{j\in I_-}
	r_j\log(q_j-\lambda) -
		\sum_{j\in I_0} r_j\log\left(u_j{\mu}{\lambda}\right) ~ ~ .
  \label{validation1:eqn}
\end{eqnarray}
Since the terms $u_i$ does not depend on $\lambda$, they can be omitted from
$L_r(\lambda)$ without changing the minimizer, so that the last term can be
replace with $-\sum_{j\in I_0}r_j \log({\mu}{\lambda})$.  Let us now examine
the product $\mu\lambda=\mu/\nu$. The latter ratio is
constant between path change points, and satisfies~(\ref{mnlin:eqn}) which
implies that
$$\frac{\mu}{\nu} = \frac{\mathcal{Q} - \mathcal{M}/\nu}{\mathcal{U}}
= \frac{\mathcal{Q} - \mathcal{M} \lambda}{\mathcal{U}}
~ ~ . $$
We can further omit the constant $\mathcal{U}$ and replace the
last term in $L_r(\lambda)$ with 
$$
-\sum_{j\in I_0}r_j \log(\mathcal{Q}-\mathcal{M}\lambda)
$$
without changing the minimizer.
To recap, between each consecutive points $(1/\nu_{i+1}\,,\,1/\nu_{i})$
the validation set likelihood, up to constant terms, is equal to,
$$
L_{\v{r}}(\lambda)=  -\sum_{j\in I_+} r_j \log(q_j + \lambda) -
                      \sum_{j\in I_-} r_j\log(q_j-\lambda) -
                      \sum_{j\in I_0} r_j
                      \log\left(\mathcal{Q} - \mathcal{M}\lambda \right) ~ ~ .
$$
It is easy to check that the second order derivative of
$L_{\v{r}}(\lambda)$ is positive, namely, it is convex. To recap,
we have the following simple picture.
While the validation objective we wish to minimize over
$[0,1]$ is not convex, it is piecewise convex: specifically, it is convex
on any interval where $\nu\mapsto \mu(\nu)$ is linear, namely between any two
path change points, where the sets $I_\pm,I_0$ are constant. See also
Fig.~\ref{cv:fig}.
 
\begin{figure}[t] \label{cv:fig}
  \centerline{\psfig{figure=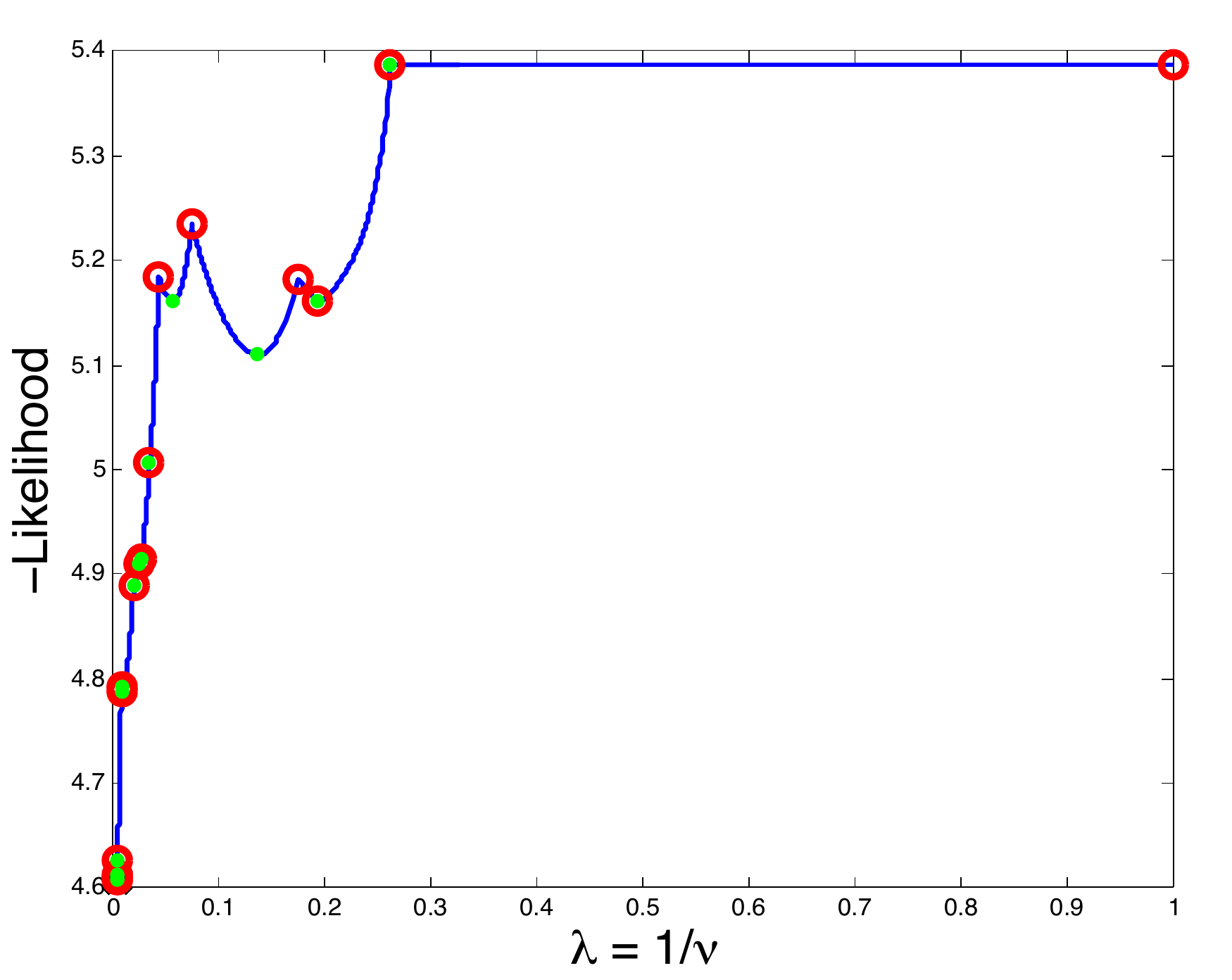,width=9cm}}
	\vspace{-0.25cm}
	\caption{The validation set likelihood is a piecewise convex function
	(Red: path nodes, green: minimum in each path segment)}
\end{figure}

It remains to show how to find the optimum in each interval numerically.
Concretely, for every consecutive path change points $\nu_i < \nu_{i+1}$ we wish
to solve
\begin{align*}
	\min_{\frac{1}{\nu_{i+1}}\leq \lambda \leq \frac{1}{\nu_i}}
        L_{\v{r}}(\lambda)=  -\sum_{j\in I_+}r_j \log(q_j + \lambda) -
        \sum_{j\in I_-} r_j\log(q_j-\lambda)
-\sum_{j\in I_0} r_j \log\left(\mathcal{Q} - \mathcal{M}\lambda \right).
\end{align*}
Since the objective $L_{\v{r}}(\lambda)$ is convex on a real interval, the
minimum exists, is unique, and can be found by bisection. In each interval,
we search for the $\lambda$ for which
$\frac{d}{d\lambda}L_{\v{r}}(\lambda)=0$. The derivative of $L_r(\lambda)$
\begin{align*}
  L'_{\v{r}}(\lambda) =  -\sum_{j\in I_+}r_j \frac{r_j}{q_j+\lambda} +
	  \sum_{j\in I_-} r_j \frac{r_j}{q_j-\lambda}
  + \left( \frac{\mathcal{M}}{\mathcal{Q}-\mathcal{M}\lambda} \right)
	   \sum_{j\in I_0} r_j ~ ~ ,
\end{align*}
is an increasing function in $\lambda$. Thus we simply need to search by
bisection for the zero crossing of $\frac{d}{d\lambda}L_{\v{r}}$.

\paragraph{A discrete set of admissible models.}
We have thus revealed a simple and aesthetic trade-off between model
complexity and validation loss for the relaxed maximum entropy problem. Having
solved for the path $\nu\mapsto\mu(\nu)$, we obtain a collection of intervals,
on each of which the path is linear. In each interval, the support of the dual
variable $\v{\alpha}$ is constant. The support size that is thus associated
with each path interval represents the number of entries in which the solution
$\v{p}$ is not proportional to the base distribution $\v{u}$, or in other
words, the model complexity. After finding the minimum cross-validation loss
at each interval, we can further associate with each path interval a single
validation loss $L_{\v{r}}(\nu^*)$ (the validation loss minimum over that
interval) and a single value $\nu^{*}$ (where this minimum is obtained).  We
thus obtain a discrete list of possible models, where each model is expressed
by its size, a regularization value, and a validation likelihood.

\begin{table}[t]
  \centering
  \begin{tabular}{|c|c|c|}
\hline
$\text{supp}(\v{\alpha}) $  & $\nu^{*}$ &
    $L^{*}_{\v{r}}$ \\
\hline
    $0$ & $\nu^{*}_0=1$ & $L^*_{\v{r}}(1)$ \\
    $1$ & $\nu^{*}_1$ & $L^*_{\v{1}}(\nu^*_1)$ \\ 

$\vdots$ & $\vdots$ & $\vdots$ \\ 
    $k$ & $\nu^{*}_k$ & $L^*_{\v{r}}(\nu^*_k)$ \\ 

    \hline
  \end{tabular}
  \caption{An illustration of the discrete list of possible parameters for
	a relaxed maximum entropy problem, given a validation set.}
  \label{tab:cv-list}
\end{table}

Since increasing model complexity only makes sense if it reduces the
validation loss, we can cull the list and keep only the options where the
validation loss decreases with the increase in the support of $\alpha$. Further,
we need only consider the cases where the support of $\alpha$ ranges from
$0$ (where the regularization penalty is infinite and $\v{p}=\v{u}$) to the
model complexity that globally minimizes the validation loss. Formally, we
obtain table of the form of Table \ref{tab:cv-list} which contains an
increasing list 
$\nu^*_0=1 < \nu^*_1 < \ldots < \nu^*_k $ such that
$\nu^*_k = \text{argmin}_{\nu\geq 1}L_{\v{r}}(\nu)$,
and such that for each $j=0,\ldots,k$,  
\begin{eqnarray*}
\nu^*_j = \arg\min_{\nu\,:\,\text{supp}(\v{\alpha}(\nu))=j}
L_{\v{r}}(\nu) \,,
\end{eqnarray*}
where $\alpha(\nu)$ is the dual solution of $\v{p}(\nu)$ and
$\text{supp}(\v{\alpha}(\nu))$ is the number of non-zero values in $\v{\alpha}$.

\iffalse
In summary, in the presence of a validation vector of counts $\v{r}$, we
derived an efficient algorithm that reduces the family of admissible models to
the discrete list as illustrates in Table \ref{tab:cv-list}. The table
encompasses a simple trade-off between model complexity and validation loss.
This list makes explicit the benefit of increasing model complexity in terms
of improvement in validation loss.
\fi

\section{Empirical Complexity of the Relaxation Path} \label{eval:sec}
In this section we assess the empirical complexity of the
relaxation path, in terms of number of path nodes.  As we will show, in
practice, the number of change points
is close to linear in the dimension. This renders the local tracking algorithms
viable approaches for real datasets. We consider two examples, one
with synthetic data and the other with data from natural text. It is
worth noting though that we obtained qualitatively similar results in all of
our experiments.

\begin{figure}[t]
  \centerline{\psfig{figure=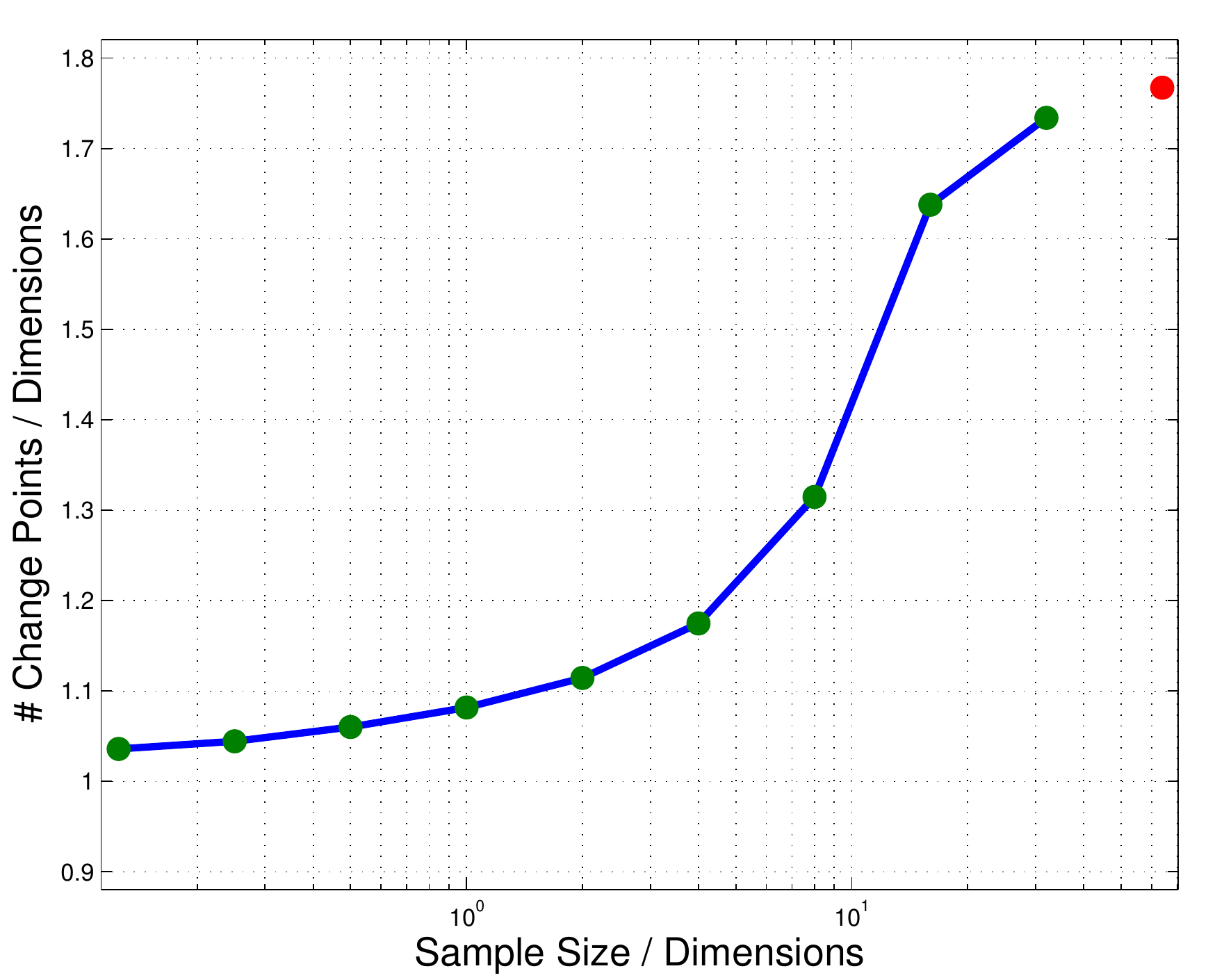,width=8cm}}
	\vspace{-0.25cm}
	\caption{Homotopy complexity as a function of the sample size.}
	\label{homotopy_complexity:fig}
\end{figure}
In the first experiment the prior distribution $\v{u}$ was a Zipf distribution
of the form $u_j \sim \frac{1}{2+j} $. The observed distribution $\v{q}$ was
sampled from a Zipf distribution which we denote by $\bar{\v{q}}$ where
$\bar{q}_j \sim \frac{1}{j}$. The dimension was set to 50,000 and we
normalized $\v{u}$, $\bar{\v{q}}$, and the sampled distribution $\v{q}$ so
that they sum to $1$. We generated numerous observed distributions $\v{q}$ by
sampling from $\bar{\v{q}}$. For each sample size we generated $10$
distributions from $\bar{\v{q}}$ and ran the local tracking algorithm with
$\v{u}$ and the sampled distribution $\v{q}$. We then computed the mean over
the 10 runs of the relaxation path complexity. Since the samples varied in
size, $\v{q}$ tended to have more zero entries the smaller the sample size
was. In Fig.~6 we show the relaxation path complexity in terms of the number of
changes points, divided by the dimension, as a function of the sample size
(also divided by the dimension). Evidently, the relaxation path complexity
grows almost linearly as a function of the sample sizes we examined. We also
checked the complexity without sampling by setting $\v{q} = \bar{\v{q}}$.  The
last point on the graph, which we kept disconnected, designates the complexity
when using $\bar{\v{q}}$ as the observed distribution. It is clear that for
much larger samples the complexity asymptotes at less than $1.8 n$.  In fact,
we never observed complexity above $2n$ in our experiments with Zipf
distributions.

\begin{figure}[t]
	\begin{center}
		\begin{tabular}{cc}
			\psfig{figure=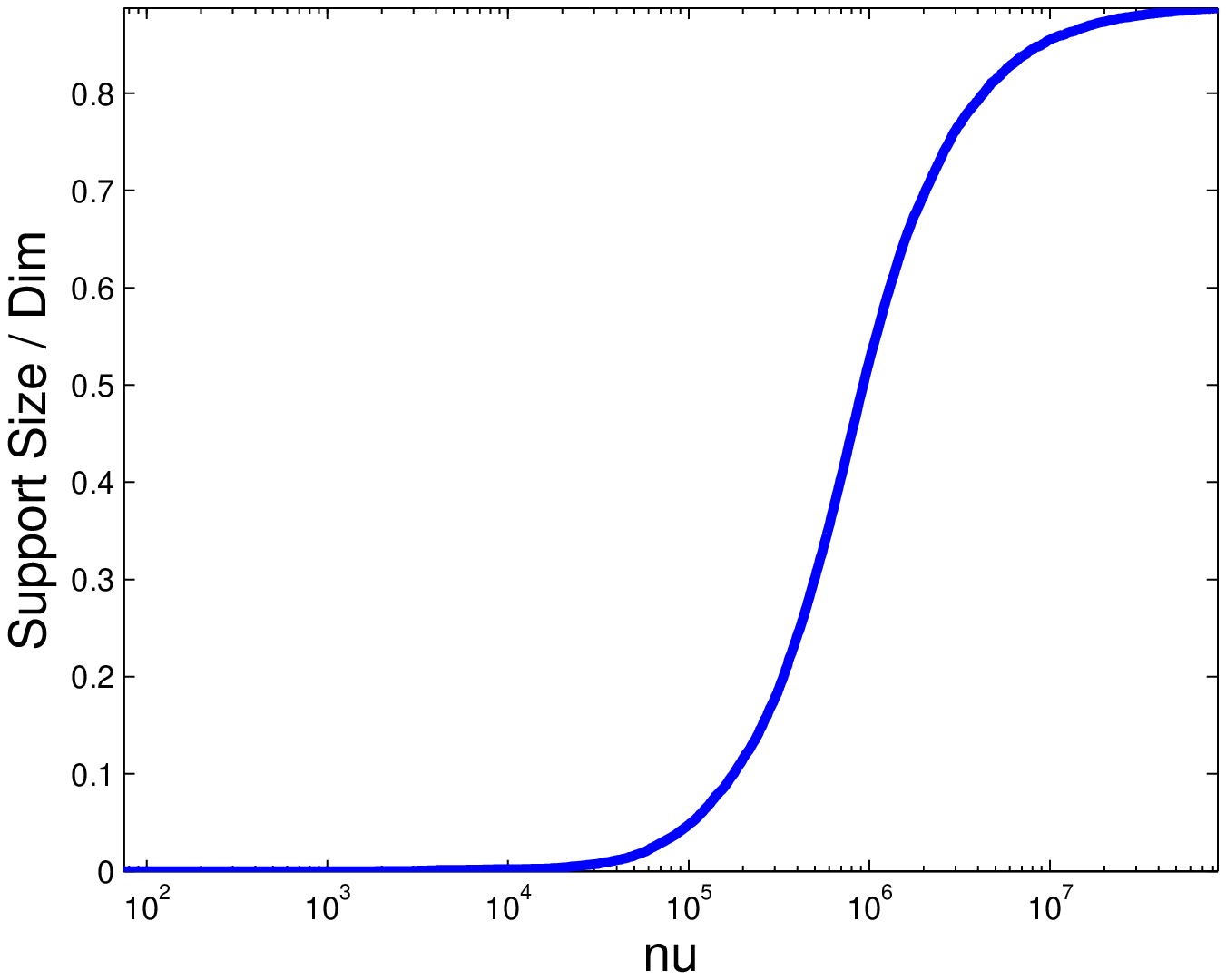,width=6cm}
			\psfig{figure=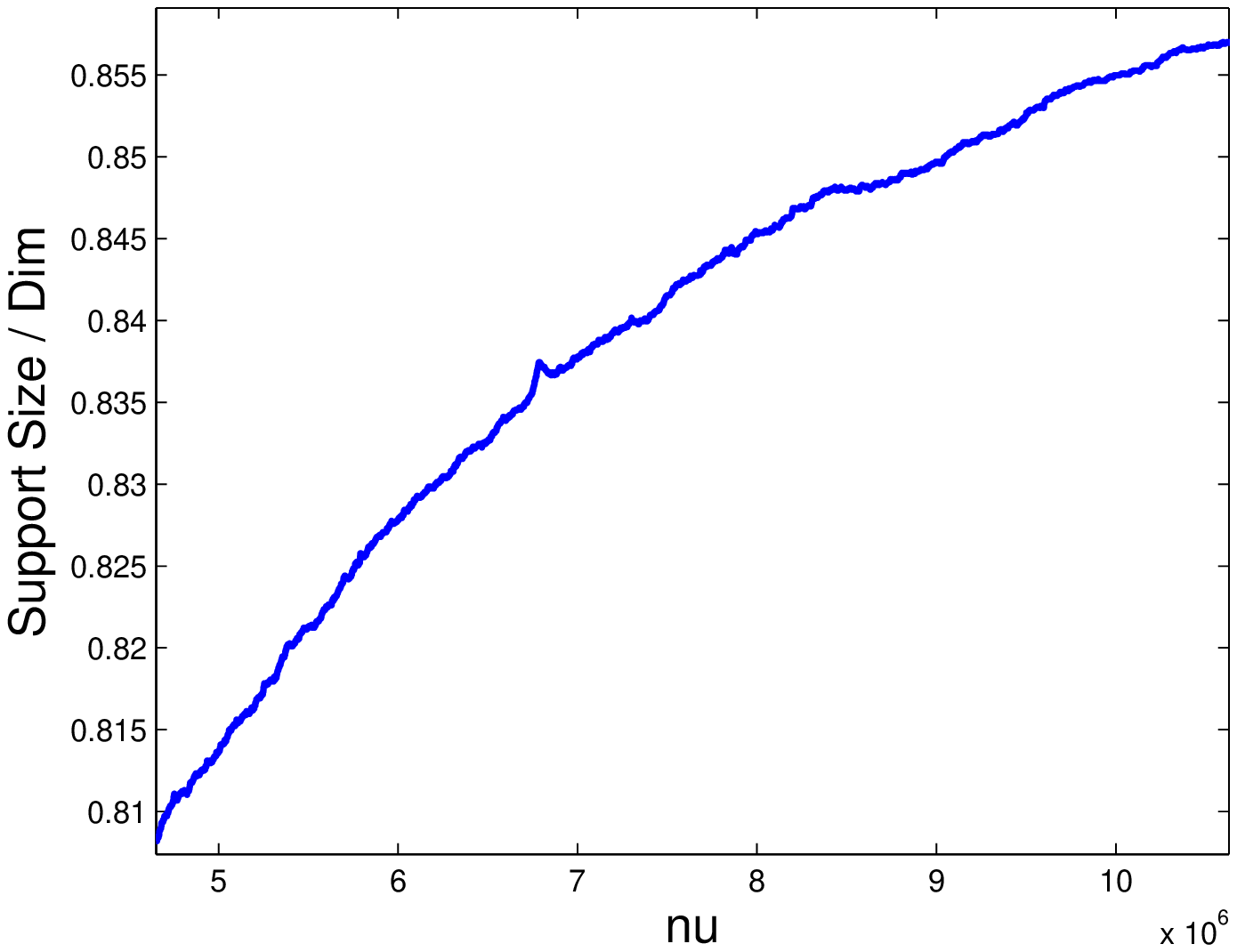,width=6cm}
		\end{tabular}
	\end{center}
	\vspace{-0.5cm}
	\caption{The cardinality of $I_+ \cup I_-$ as a function of $\nu$. Right: a
	zoomed in look into the middle range.}
	\label{support_size:fig}
\end{figure}
In the second experiment we used the Reuters Corpus Volume 1 (RCV1). The RCV1
dataset consists of a collection of approximately 800,000 text articles, each
of which is assigned multiple labels. There are 4 high-level categories,
Economics, Commerce, Medical, and Government (ECAT, CCAT, MCAT, GCAT), and
multiple more specific categories. After stemming and stop-listing we obtained
a dictionary of 17,632 words for the entire collection. We used the
maximum likelihood estimate of the entire collection to define the prior
distribution and the collection resulting by taking a single category as
the observed distribution. This setting mimics the source adaptation setting
of~\citep{BlitzerDrPe07} in which we need to adapt a distribution derived from
a large collection to a concrete distribution. In
Fig.~\ref{support_size:fig} we show the size
of the support, i.e. $|I_+ \cup I_-|$ as a function of $\nu$ for one of the
categories (ECAT). We would like to note that the path complexity was,
20,628, just barely more than the dimension (dictionary size). We plot on the
left hand side of the figure in log-scale the support size as a function of
$\nu$. It seems that the support is monotonic in $\nu$ and there is a region
where it grows linearly. However, a closer examination (right figure)
reveals non-monotonic changes in the size of $|I_+ \cup I_-|$. Examining the
individual sets we see that the overall tendency is to increase their size but
``locally'' an index may enter and exit the sets $I_+$ and $I_-$ multiple
times. This behavior further justifies the usage of the local tracking
algorithm while ruling out approximate algorithms that enforce monotonicity
of the sets $I_+$ and $I_-$.

\section{Application to n-gram Models} \label{ngram:sec}
In this section we demonstrate the potential of the maximum entropy relaxation
path for building an n-gram model. We would like to note though that
the goal of this section is to underscore the potential of the approach rather
than to obtain state-of-the-art results on a concrete benchmark.  We refer the
reader to \citep{chen1999empirical} for an introduction to n-gram
models. A context is a string $s=\omega_n \omega_{n-1}\ldots \omega_1$ where
each token $w_i$ is from to an alphabet $\Omega$. The tokens may be, for
example, words in a dictionary, phonemes or, characters. In an n-gram model,
the conditional probability $p(\cdot \, | \, \omega_n \ldots \omega_1)$ over
$\Omega$ is modelled using a variable-length Markov chain of fixed maximal
length \citep{RonSiTi96,buhlmann1999variable}. 

A convenient way to describe such models is by a suffix tree, whose nodes are
tokens, such that paths from root to leaf determine a context $\Omega$ (see
Figure~\ref{observed:subfig}). Modern n-gram models are typically trained
using a two-step procedure. First, a training buffer is used to count the
number of occurrences of tokens in any given context: define $c(\omega |
\omega_n \ldots \omega_1)$ to be the number of occurrences of the string
$\omega_n \ldots \omega_1 \omega$ in the training buffer. These counts can are
normalized to empirical distributions $q(\omega | \omega_n \ldots \omega_1)$
over $\Omega$ and described by a function on the suffix tree (see again
Figure~\ref{observed:subfig}). These empirical distributions are just the
maximum likelihood estimators for the underlying distributions on $\Omega$.  A
parameter estimation of an n-gram model faces a classical bias-variance
trade-off: each empirical distribution in a certain context is more accurate
but also has a higher estimation variance than the empirical distribution in
its shorter, parent context. To mitigate the variance of the estimates with
the increase of the context length a second step, known as a {\em smoothing}
of {\em back-off} is used. The second step is a specific recipe to combining or
interpolating empirical distributions from different levels along each path
from root to leaf on the suffix tree. The smoothing procedure often concludes
with a secondary subprocess, called pruning, in which contexts whose estimated
distribution is suspected to be inaccurate are removed from the final
estimated model (Figure~\ref{pruned:subfig}).
\begin{figure}[t] \label{tree:fig} \centering
	\subfigure[Observed]{\hspace{10mm}\psfig{figure=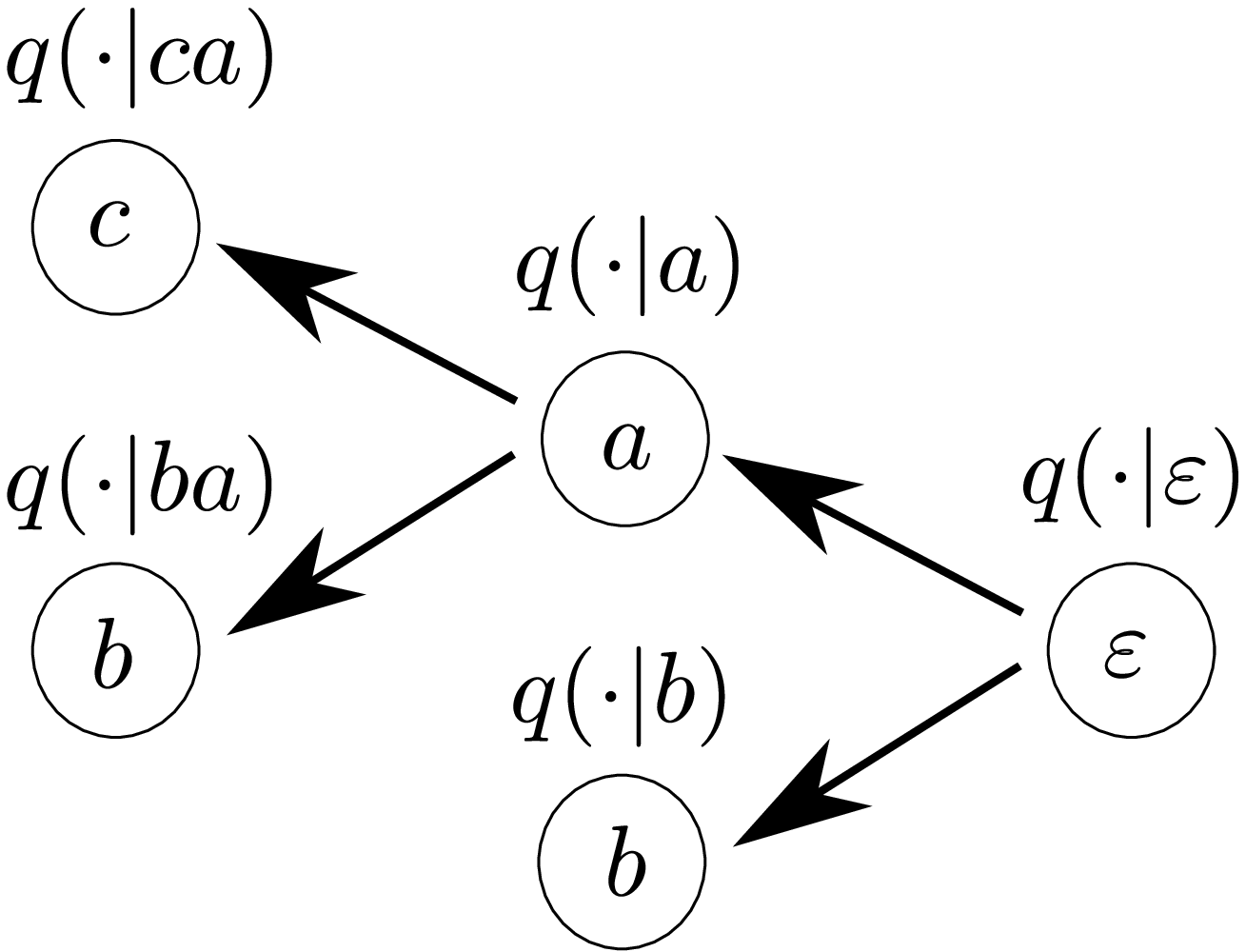,width=6cm}
	\label{observed:subfig}}
	\subfigure[Estimated and
	pruned]{\hspace{10mm}\psfig{figure=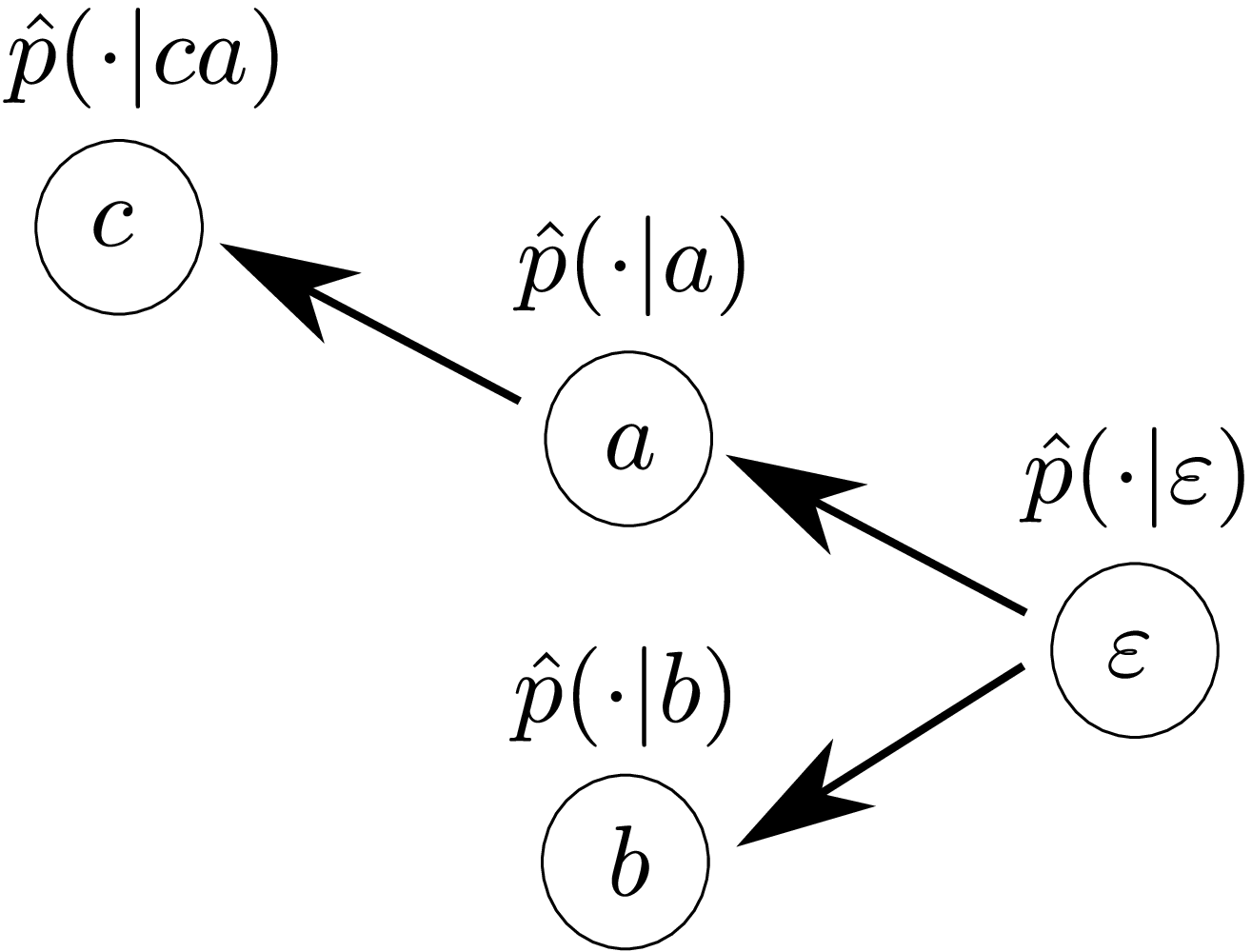,width=6cm}
	\label{pruned:subfig}} \caption{Estimation of an n-gram model represented by
	a suffix tree} \end{figure}

N-gram models for machine translation and speech recognition, where
the alphabet $\Omega$ is a significant subset of a natural language
dictionary, can get quite large in terms of number of
parameters~\citep{Brants07largelanguage}. In trying to reduce the storage and
retrieval burden imposed by prediction using a very large n-gram model,
several compressed representations of n-gram models in general, and
for those represented using suffix trees in particular, have been proposed,
see for instance~\citep{pauls2011faster}. Alternatively, instead of asking to
reduce the number of parameters by compacting an estimated model, one might
wish to leave out parameters, whose inclusion in the model does not give
sufficient improvement in performance. This idea leads to a procedure
known as entropy-based pruning \citep{stolcke2000entropy}.  In this section we
propose a new approach in which the estimation and pruning procedures are two
facets of the same process.  Instead of pruning a fully estimated n-gram
model, ideally we start with a given {\em parameter budget} and allocate
parameters from the budget in the most advantageous way, as measured by
likelihood of the model on a validation buffer. As we now show, the
regularized maximum entropy problem offers a natural approach to the parameter
allocation in an n-gram model.

\paragraph{Cascading relaxed maximum entropy problems.} A first step towards an
n-gram model based on relaxed maximum entropy problems is cascading a sequence
of maximum entropy problems, where each problem receives as its ``prior''
distribution the solution of the problem that precedes it. Assume that we have
a sequence of ``empirical'' distributions $\v{q}^{(1)} \dots \v{q}^{(s)}$ and
a ``prior'' distribution $\v{u}$ on $\Omega$.  We cascade maximum entropy
problems by taking, for the $i+1$-th problem, a specific solution
$\v{p}^{(i)}$ to the $i$-th problem as ``prior'' distribution  and
$\v{q}^{(i+1)}$ as ``empirical'' distribution. Formally, for $i\geq 1$ let
\begin{eqnarray*}
  \v{p}^{(i)} = \text{argmin}_{\v{p}\,\in\simplex}
  \sum_{j=1}^n p_j\log\!\left(\frac{p_j}{p^{(i-1)}_j}\right)
  ~ \mbox{ s.t. } ~ \|\v{p}-\v{q}^{(i)}\|_\infty \leq 1/\nu^{(i)}
  ~ ,
\end{eqnarray*}
where $\nu^{(i)}$ is our choice of relaxation parameter for the $i$-th
problem. It is convenient to write $\v{p}^{(0)}=\v{u}$ and $Z^{(0)}=1$. 
Denoting by $\v{\alpha}^{(i)}$ the dual solution to $\v{p}^{(i)}$, we have 
$$p^{(i)}_j = \frac{p^{(i-1)}_j e^{\alpha^{(i)}}}{Z^{(i)}} 
\mbox{  where  } Z^{(i)}=\sum_j p^{(i-1)}_j \exp\left(\alpha^{(i)} \right) ~ .$$
The $i$-th solution thus has the form
$$
p^{(i)}_j = \frac{ \exp\left( \sum_{k=0}^{i}\alpha_j^{(k)}
\right)}{\prod_{k=0}^{i}Z^{(k)}}\,.
$$
Consequently, to store the distributions $\v{p}^{(1)},\ldots,\v{p}^{(s)}$, we
need only store the dual solutions $\v{\alpha}^{(1)},\ldots,\v{\alpha}^{(s)}$,
which tend to be sparse, and the normalization factors
$Z^{(1)},\ldots,Z^{(s)}$.

% Macros for this section only
% ----------------------------
% probability vector for context
\newcommand{\vcon}[2]{\v{#1}(\cdot | \omega_{#2} \ldots \omega_1)}
% entry for probability vector in context
\newcommand{\con}[2]{#1(\omega | \omega_{#2} \ldots \omega_1)} 

\paragraph{Using the relaxation path for parameter allocation.} We now
describe a framework for parameter allocation and pruning in n-gram models. We
use cascading of relaxed maximum entropy problems along paths from the root to
leaf of a suffix tree, combined with the model selection procedure of
Sec.~\ref{cv:sec}.  To set the notation, let us write $\vcon{p}{n}$ for the
estimated probability distribution in the context $\omega_n \ldots \omega_1$.
This is a vector indexed by the token $\omega\in\Omega$, and we write
$\con{p}{n}$ for the value it takes at $\omega$.  Similarly, let $\vcon{q}{n}$
denote the empirical distribution vector as calculated over the training
buffer, let $\vcon{r}{n}$ be the empirical distribution vector in this context
as calculated over a validation buffer, and let $\v{p}(\cdot|\varepsilon)$
denote a given ``prior'' where $\varepsilon$ denotes the empty context.

Suppose that we have determined the probability distribution $\vcon{p}{n}$,
which our model uses for prediction in the context $\omega_n \ldots \omega_1$.
The task is to determine the distributions $\vcon{p}{n+1}$ in each of the
sub-contexts of length $n+1$. For each sub-context $\omega_{n+1} \ldots
\omega_1$, we solve for the entire relaxation path of the relaxed maximum
entropy with ``prior'' distribution $\vcon{p}{n}$ and ``observed''
distribution $\vcon{q}{n+1}$.  With the relaxation path available, we perform
the efficient cross validation procedure of Sec. \ref{cv:sec} against the
validation distribution $\vcon{r}{n+1}$, and obtain a list of options in the
form of Table \ref{tab:cv-list}.  For the sub-context $\omega_{n+1} \ldots
\omega_1$, write $\left( k^{\omega_{n+1}}, L(k^{\omega_{n+1}})\right)$ for the
option with model size $k^{\omega_{n+1}}$ and corresponding validation loss
$L(k^{\omega_{n+1}})$.  (The context  $\omega_n \ldots \omega_1$ is held fixed
and implicit in this notation.) For each sub-ontext, each option in the list
corresponds to a specific choice of prediction distribution in this context. To
choose one option out of each list, an allocation rule is applied as discussed
below.  Let $\kappa^{\omega_{n+1}}$ denote the option chosen by the allocation
rule for sub-context $\omega_{n+1} \ldots \omega_1$.  This results in a
specific prediction distribution $\vcon{p}{n+1}$ for each sub-context, and
decreases the total parameter budget by
$\sum_{\omega_{n+1}\in\Omega}\kappa^{\omega_{n+1}}$. The remaining parameter
budget is divided among the sub-contexts according to the allocation rule,
and the algorithm proceeds recursively on each sub-context that received an
allocation $\kappa^{\omega_{n+1}}>0$, each with its own budget.
The end result is a compact exponential representation of a recursive form,
\begin{eqnarray*}
\con{p}{n+1}
& \eqdef &
\frac{ \con{p}{n} exp\left( \con{\alpha}{n+1} \right)  }{\con{Z}{n+1}} \\
&=&  p(\omega|\varepsilon) \frac{exp\left(\sum_{i=1}^{n+1} \con{\alpha}{i}
\right)}{\prod_{i=1}^{n+1}\con{Z}{i} } \,,
\end{eqnarray*}
where $\con{Z}{n+1}$ is the required normalizing factor and
$\con{\alpha}{n+1}$ is the dual solution to $\con{p}{n+1}$. The sparsity of
$\vcon{\alpha}{n+1}$, namely the amount of parameters allocated to
specializing $\vcon{p}{n}$ to $\vcon{p}{n+1}$  is chosen to give the best
improvement in validation loss per parameter.  We can thus estimate and store
$\vcon{p}{n+1}$  not as a distribution in its own right, but as an exponential
tilt of the previously estimated distribution in the shorter context.

\paragraph{Parameter Allocation and Pruning.}
Unlike most well-known backoff or smoothing methods for training n-gram
models, the approach proposed here does not contain a separate pruning
procedure \citep{kneser1996statistical} that is executed after the model has
been trained in full. Instead, pruning happens naturally as part of parameter
allocation. If the allocation rule set $\kappa^{\omega_{n+1}}=0$ for the
sub-context $\omega_n+1 \ldots \omega_1$, then this sub-context (as well as
the sub-suffix-tree rooted at it) were pruned.  The allocation $\kappa^\omega$
to a certain sub-context $\omega \omega_1 \ldots \omega_n$ may be understood
as follows: according to the validation set, the benefit that can be expected
from letting the distribution $\vcon{p}{n+1}$ differ from $\vcon{p}{n}$ is not
worth allocating even a {\em single} parameter. We use a simple allocation
rule in our experiments. We fixed a maximal tree depth $N$ and a per-node
budget $b$. We determined the allocation of each sub-context independently of
the other sub-contexts. For sub-context $\omega_{n+1} \ldots \omega_1$, if
$N=n$, allocate $\kappa^{\omega_{n+1}}=0$.  Otherwise, choose the option
$k^{\omega_{n+1}}$ with most parameters such that $k^{\omega_{n+1}}\leq b$.
Namely, we allocated at most $b$ parameters at each node (single maximum
entropy problem). This allocation is rather rudimentary and its goal is mostly
to demonstrate the potential of the maximum entropy relaxation path.
\begin{figure}[t] \label{ngram-results:fig}
  \centering
  \subfigure[English]{\hspace{10mm}\psfig{figure=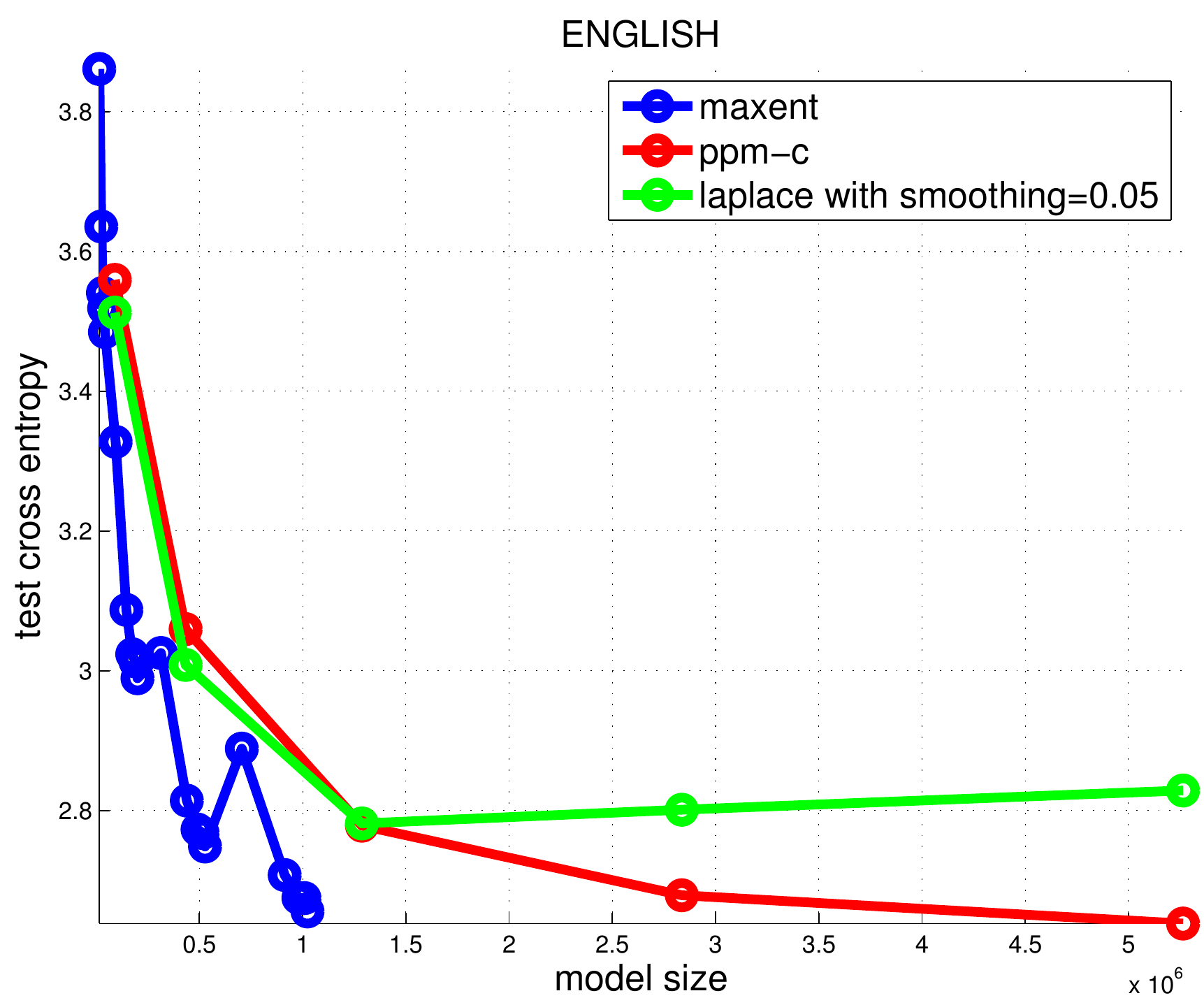,width=6cm}
  \label{english:subfig}}
  \subfigure[Arabic]{\hspace{10mm}\psfig{figure=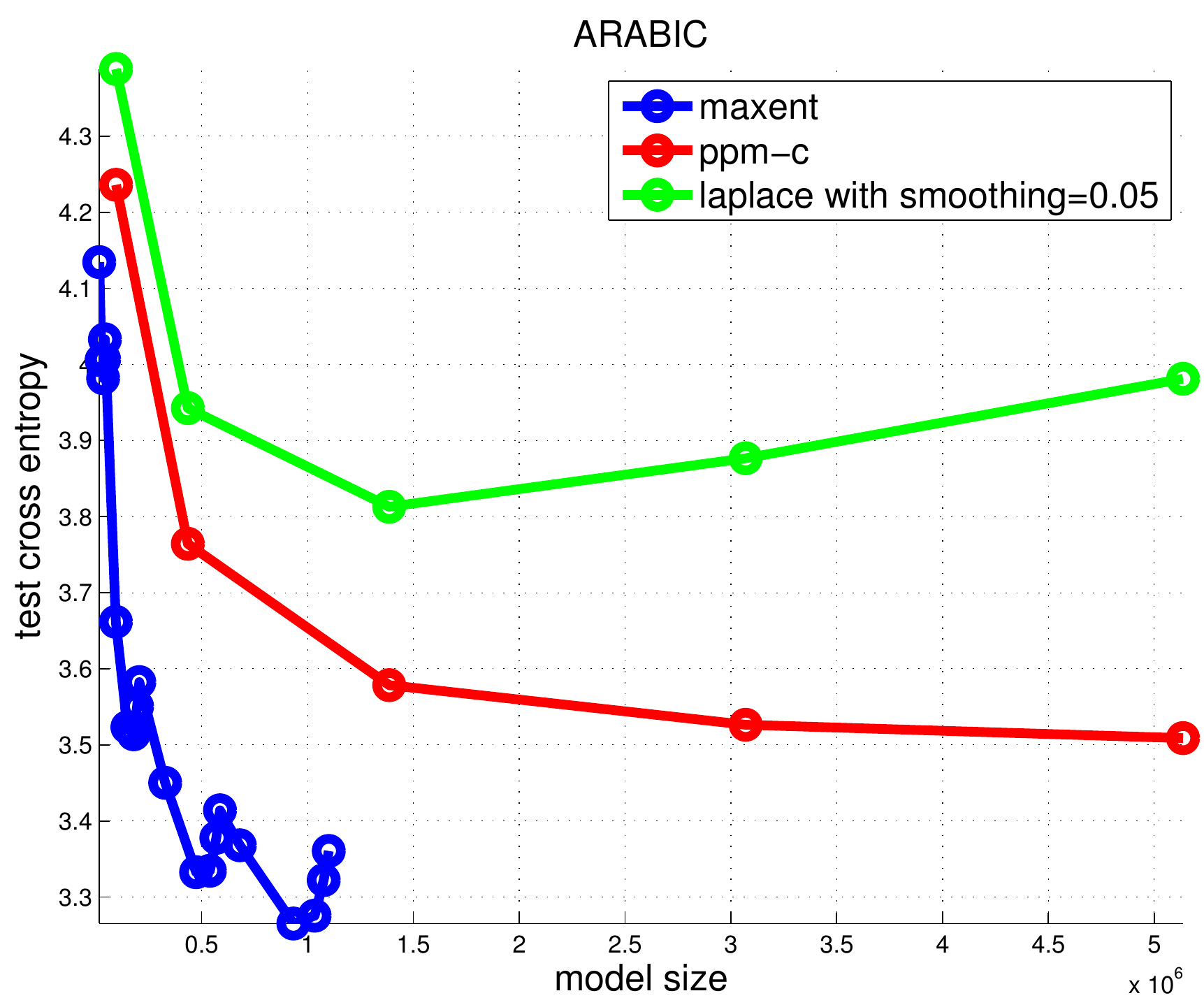,width=6cm}
  \label{arabic:subfig}}
  \subfigure[Hindi]{\hspace{10mm}\psfig{figure=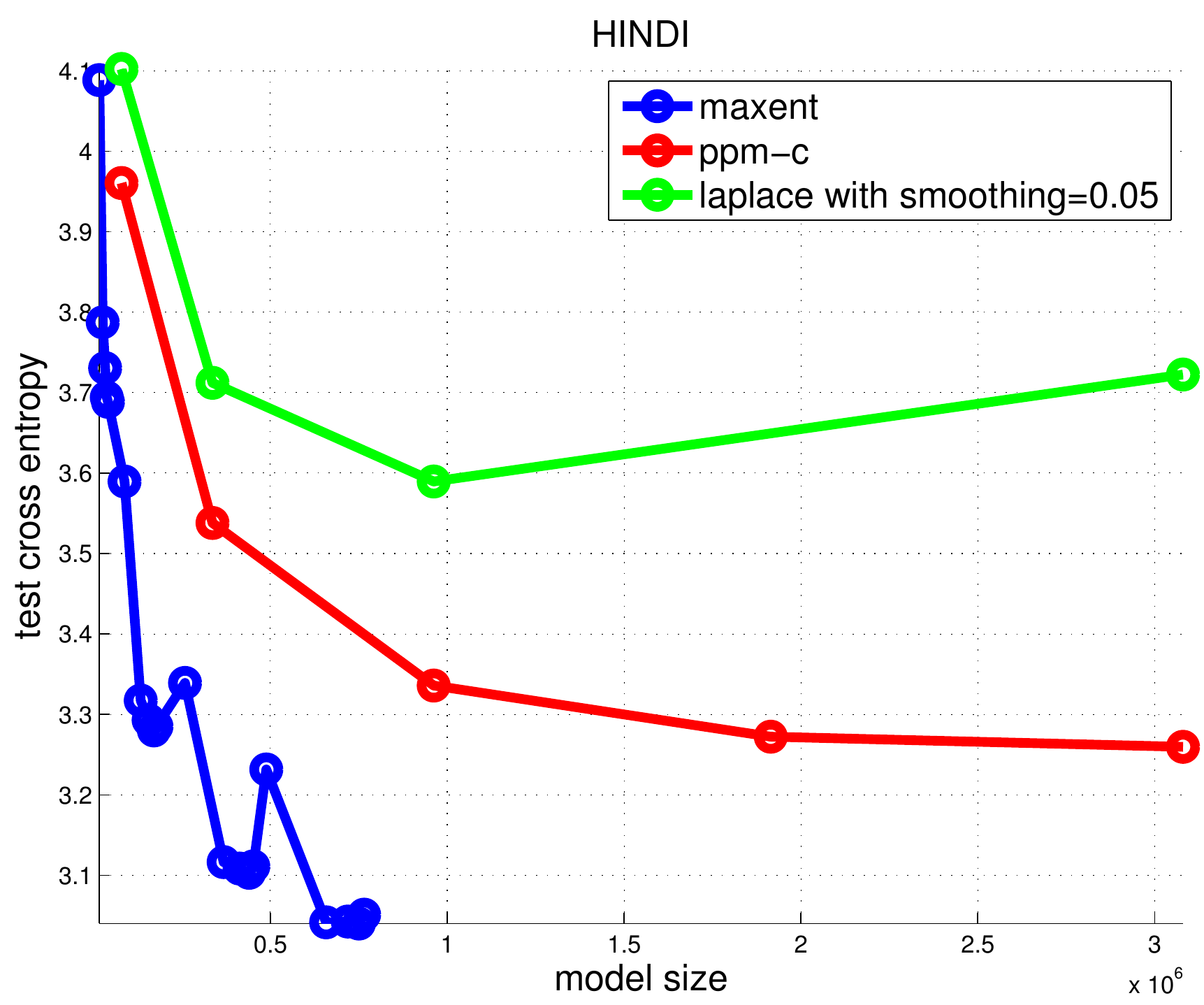,width=6cm}
  \label{hindi:subfig}}
  \subfigure[Russian]{\hspace{10mm}\psfig{figure=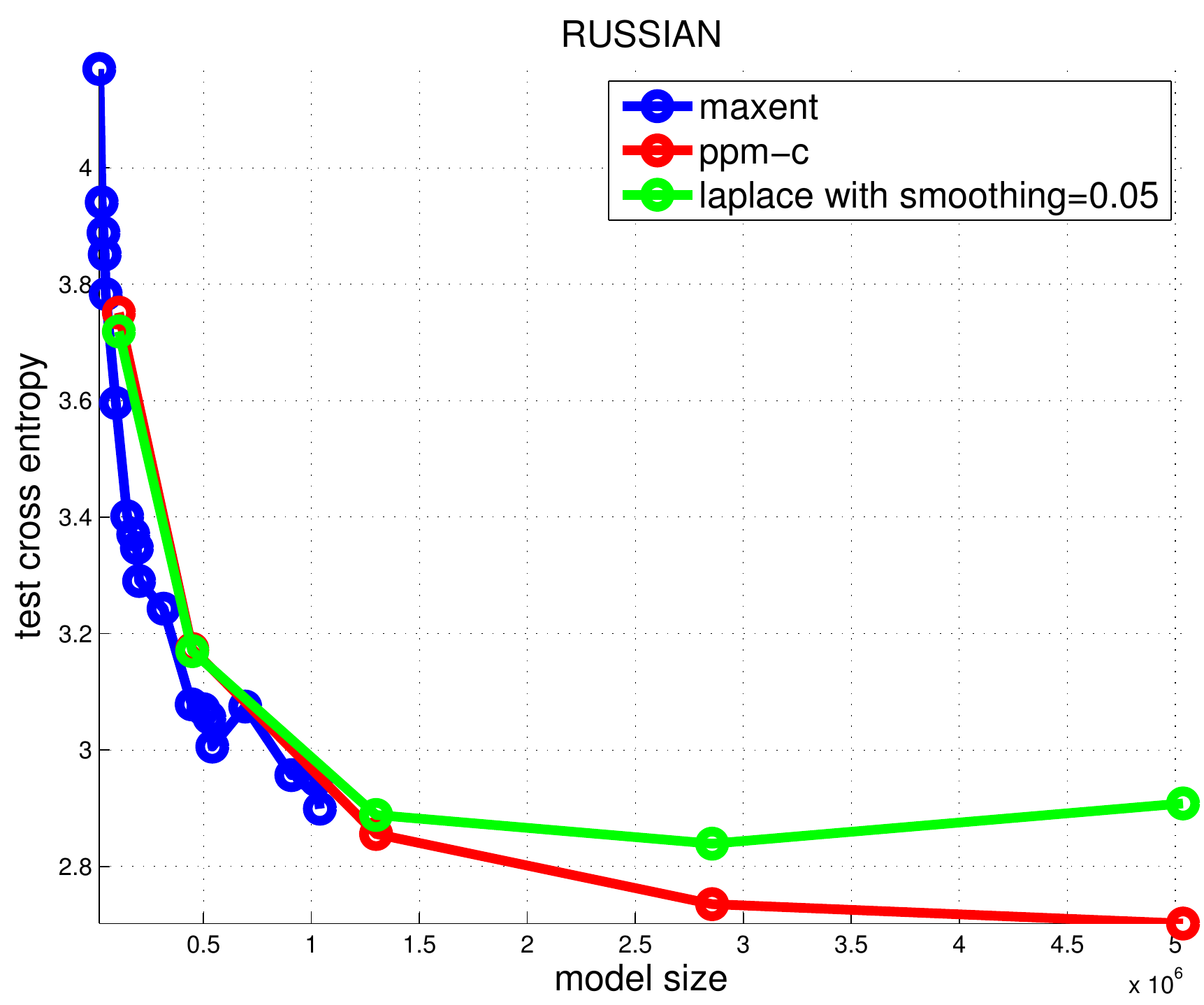,width=6cm}
  \label{russian:subfig}}
  \caption{Character-based n-gram models:
	Comparison of cross-entropy performance on unseen test buffer}
\end{figure}

\paragraph{Evaluation.} We have implemented our approach in a prototype
character-based n-gram
model. Such models, in which the alphabet $\Omega$
consists of a subset of the Unicode symbols, are commonly used in optical
character recognition (OCR) applications in conjunction with an image
processing module. The predicted probability provided by the n-gram model is
combined with an image content score in order to determine the identity of the
characters composing the image.  Depending on the language, the alphabet size
is on the order of thousands. The available data was a buffer of $40\cdot 10^6$ 
characters from several languages. The last $20\%$ of each buffer were held
out and used as the validation set. Each n-gram model was trained according to
the maximum entropy framework with the trivial allocation above. Different
model sizes were achieved by tuning the maximal tree depth and the maximum
parameter allocated to any individual maximum entropy problem. Performance of
each model was measured using cross entropy against a new test buffer with
approximately $1\cdot 10^6$  characters.  As a baseline, performance was compared
against a Laplace (or ``add-one'') smoothing model with a smoothing parameter
held fixed at $0.05$. Performance was compared against PPM-C
\citep{moffat1990implementing}, one of the current method of choice for
training character-based n-gram models. We capped the maximal depth of the
trees obtained by Laplace smoothing and PPM-C in order to control their model
size. The results are summarized in Figure \ref{ngram-results:fig}. The plots
show the trade-off of model size versus performance. We do not expect that the
method proposed here would outperform state of the art when the model size is
unlimited. Rather, we expect it to provide a favorable size versus performance
trade-off, particularly in the regime where model size is severely restricted.
Our results seem to support this hypothesis. Indeed, we see that especially
for Hindi and Arabic we can achieve low out of sample entropy for very small
model sizes.

\iffalse
We remark that training of the maximum entropy-based n-gram model for
translation language models, where $\#\Omega\approx 10^6$ is quite feasible.
The observed vectors $\vcon{q}{n}$ are sparse, and grow extremely sparse
quickly with the context length. Hence, using the fast tracking algorithm
discuss in \ref{local-homotopy-sparse-obs:subsec}, even a small machine can
easily recover the entire path and perform cross-validation of a single relaxed
maximum entropy problem.
\fi

\section{Extensions} \label{extensions:sec}

In the previous sections we focused on tracking algorithms for the important
case where the separable objective function is the relative entropy.  In this
section we discuss a simple but useful extension of the problem that builds on
the multiplicity vector $\v{m}$. In addition, we describe an adaptation of our
tracking algorithm for another setting where the master equation
$G(\nu,\mu)=0$ can be solved efficiently, namely the squared Euclidean
distance.

\paragraph{Coordinate-weighted relaxation.}
The observed distribution, as its name implies, is often a rational
distribution obtained through dividing the number of observations of each
event from the multinomial distribution by the total number of observations.
It is therefore desirable to impose different constraints on the differences
$|p_j-q_j|$ so as to accommodate the number of observations. For instance,
when $q_j=0$, we simply have not observed any events corresponding to the
$j$-th outcome. In this case, it is typically desirable to enforce a stricter
difference on $|p_j-q_j|$. Formally, we would like to associate an
a priori accuracy parameter $\delta_j$ such that
\begin{equation} \label{general_linf:eqn}
\forall j\in[n]: ~ |p_j-q_j|\leq \frac{\delta_j}{\nu} ~ ~ ,
\end{equation}
and perform tracking of all admissible values of $\nu$ as before.  As we
now show, we can use the multiplicity vector in place of $\v{\delta}$ to
handle this setting as long as $\v{q}$ and $\v{u}$ are in the simplex (without
multiplicity). Note while we motivated $\v{m}$ as a multiplicity vector, the
sole requirement we placed on $\v{m}$ is positivity, namely, $m_j>0$ for all
$j\in[n]$. We next set $\v{m}=\v{\delta}$ and define $\tilde{p}_j=p_j/m_j$,
$\tilde{q}_j=q_j/m_j$, $\tilde{u}_j=u_j/m_j$. The resulting problem becomes,
$$
\min_{\v{\tilde{p}}\,\in\simplex(\v{m})}
\sum_{j=1}^n m_j \tilde{p}_j\log\left(\frac{\tilde{p}_j}{\tilde{u}_j}\right)
~ \mbox{ s.t. } ~ \|\v{\tilde{p}}-\v{\tilde{q}}\|_\infty \leq 1/\nu \,,
$$
where $\v{\tilde{q}},\v{\tilde{u}}\in\simplex(\v{m}) $.
Clearly, the above problem is of exactly the same form
of~(\ref{relmaxent:eqn}). To recap, by setting the multiplicity $\v{m}$ to be
the a-priori relaxation vector $\v{\delta}$ we can straightforwardly
accommodate constraints of the form~(\ref{general_linf:eqn}) while performing
relaxation path tracking.

\paragraph{Relaxation Path for Square Loss.}
We next focus on examining a different choice for
$\v{\phi}(\cdot)$ and review the required changes to the core of the
tracking procedure. The objective function we examine is the squared
Euclidean distance between $\v{p}$ and $\v{u}$. The optimization problem is 
\begin{equation} \label{squared_diff_opt:eqn}
  \min_{\v{p}\,\in\simplex(\v{m})} \half \sum_{j=1}^n (p_j - u_j)^2
  ~ \mbox{ s.t. } ~ \|\v{p}-\v{q}\|_\infty \leq 1/\nu \,
\end{equation}
where $  \v{q},\v{u}\in\simplex(\v{m}) $.
This problem is concerned with projecting $\v{u}$ into the intersection of the
$\ell_1$ and $\ell_\infty$ polytopes. Note that by modifying $\v{m}$ we can
control the size of the $\ell_1$ ball in which $\v{p}$ should reside.
In the interest of simplicity we now assume $\v{m}\equiv 1$. The
objective function in this problem is clearly of the form given by
(\ref{separable_obj:eqn}) with $\phi_j(p_j) = \half (p_j - u_j)^2$. Recalling
the construction from Sec.~\ref{separable_homotopy:sec}, we need to calculate
$d\phi_j(p_j)/dp_j$ and find its inverse. Since
$$\frac{d\phi_j(p_j)}{dp_j} \; = \; p_j-u_j$$
the inverse function $\psi_j$ must satisfy
$$\psi_j(p_j-u_j) = p_j ~ ~ .$$ Denoting $\eta=p_j-u_j$ we get that
$\psi_j(\eta) = u_j + \eta$. Next, we can expand (\ref{ps2:eqn}) and write
$$
p_j = q_j + \frac{1}{\nu}\,\capping\left(\nu(u_j + \eta) - \nu q_j\right) ~ ~ .
$$
We define $\mu = \nu\eta$ and get that the dependency of $\v{p}$ is
piece-wise linear in $(\nu,\mu)$ as follows,
$$
p_j = q_j + \frac{1}{\nu}\,\capping\left(\nu(u_j -q_j) + \mu\right) ~ ~ .
$$
We now define $\s{M}$ as before and introduce the following definitions,
$$
\s{R} = \sum_{j\in I_0} u_j - q_j  ~ , ~ \s{B} = |I_0| ~  ~ .
$$
Then, the requirement that $G(\nu,\mu)=0$ yields = ANSWER: FIXED
the following linear
equation,
$$
\nu\s{R} + \mu\s{B} + \s{M} = 0 ~ ~ ,
$$
which forms the equation for the line $\ell_0$. We can now repeat the tracking
procedure almost verbatim and perform projections onto the intersection of the
$\ell_1$ ball with the hypercube of length $1/\nu$ centered at $\v{q}$ for all
possible values of $\nu$. This solution is an entire relaxation path
generalization of the projection procedure onto the $\ell_{1,\infty}$ polytopes
as described in~\citep{QuattoniCaCoDa09}.

%\section{Conclusions} \label{conclusions:sec}
%
%We presented an in depth study of the maximum entropy relaxation path, from
%existence and geometric description to path-tracking algorithms and
%applications. The relaxation path exists for any separable and convex
%objective functions. Apart from our main setting of the maximum entropy
%objective, we showed that the square loss objective admits similar
%characterization and tracking. Along this line, we  plan to further examine
%separable Bregman functions in order to derive entire solution paths for less
%explored objectives such as the Itakura-Saito spectral
%distance~\citep{RabinerJu93}. The maximum entropy problem, seen as a problem
%of interpolating two probability distributions on a finite set, arises in
%various applied settings. We are currently investigating the value of our
%approach in some of these applications.  The extent to which the approach
%presented here applies to convex objectives that are non-separable yet
%structured, such as the Lasso setting of~\citep{OsbornePrTu00}, is an
%interesting open problem.

\section*{Acknowledgments}
We thank John Duchi, Sally Goldman, Han Liu, and Fernando Pereira for valuable
feedback.
Thanks to John Blitzer for comments and
suggestions on the final manuscript. MG was partially supported by a William R. and Sara
Hart Kimball Stanford Graduate Fellowship and conducted this work while at
Google Research.

\appendix

\section{Global Homotopy Tracking} \label{global_homotopy:sec}
The local homotopy tracking algorithm repeatedly calculates the intersection
of $\lll_0$ with {\em all} of the lines $\lll_{\pm j}$, regardless of their
orientation. Thus, computing the intersections of $\lll_0$ with the lines
$\lll_{\pm j}$ requires $O(n)$ operations. In this section we present an
alternative algorithm that maintains a global view of all of the lines
including the line $\lll_0$. The global homotopy tracking requires though a
more complex data structure, namely a priority queue~\citep{CormenLeRiSt01},
that facilitates insertions and deletions in $O(\log(n))$ time and finding the
smallest element in the data structure in constant time.

The global algorithm scans the $(\nu,\mu)$ plane left to right, starting with
$\nu=0$, while maintaining the intersections of the lines $\lll_{\pm j}$ and
$\lll_0$ with a vertical line placed at $\nu$. We denote by $\nll_\nu$ the
vertical line located at $\nu$.  Throughout the scanning process we maintain a
priority queue for reasons that are explained shortly. We denote by $\v{\chi}$
the vector which records the order in which the lines $\lll_{\pm j}$ {\em and}
$\lll_0$ intersect the horizontal line $\nll_\nu$ (see also
Fig.~10). Formally,
the vector $\v{\chi}$ records for a given $\nu$ the vertical intersections
such that $\mu_{\chi_0} \leq \mu_{\chi_2} \leq \ldots \leq \mu_{2n}$ where $$
(\nu, \mu_{j}) = \lll_j \cap \nll_\nu ~ ~ \mbox{ for } -n\le j\le n ~ ~ . $$
Note that the line $\lll_0$ which defines the solution is treated as any other
line $\lll_j$ for $1\le |j|\le n$. Naturally, our procedure would take a few
additional steps when the scanning process is concerned with $\lll_0$, as we
describe in the sequel.

\begin{figure}[t]
	\begin{center}
  {\epsfig{figure=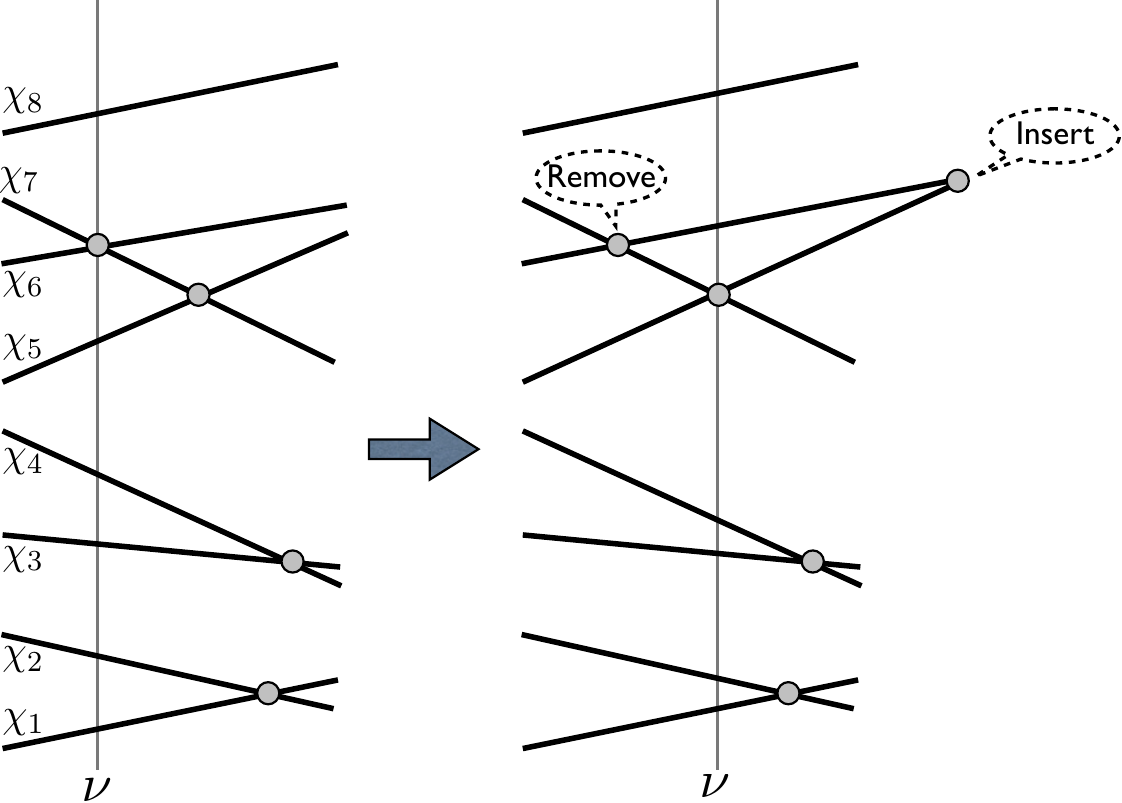,height=6cm}}
	\end{center}
	\caption{Illustration of the line scanning process during a single iteration
	of the global homotopy tracking.}
  \label{scanglob:fig}
\end{figure}

The definition of $\v{\chi}$ implies that adjacent lines on $\nll_\nu$
intersect at
$$ (\tilde{\nu}_i, \tilde{\mu}_i) =
  \lll_{\chi_{i-1}} \cap \lll_{\chi_{i}} ~ ~ , $$
where $1\le i\le 2n$. We use a priority queue to keep track of future
intersections. The intersections where $\tilde{\nu}_i\le\nu$ have
already been observed and no longer reside in the priority queue.
The pairs $(\tilde{\nu}_i, i)$ where $\tilde{\nu}_i>\nu$
are arranged in the priority queue in an increasing order,
so that the smallest $\tilde{\nu}_i>\nu$ is at the front of the priority queue.
For each element $\tilde{\nu}_i$ we also keep the index $i$ in order to
retrieve the line the intersection corresponds to in a constant time.
During the scanning process we also maintain the variables
$\s{M},\s{U},\s{Q}$ which describe the current segment of $\lll_0$.
We start at $\nu=0$, $\s{M}=0$, $\s{U}=\s{Q}=1$ and populate the priority
queue. The tracking process involves the following iterates.
We pull out the minimal pair $(\tilde{\nu}_i, i)$ from the priority queue.
The value $\tilde{\nu}_i$ corresponds to the intersection of the line
$\lll_{\chi_{i-1}}$ and $\lll_{\chi_i}$.

Let us examine first the case
where both $\chi_i\neq 0$ and $\chi_{i-1}\neq 0$. In this case the two
lines simply switch their position on the scan line $\nll_\nu$ as $\nu$
passes through $\tilde{\nu}_i$. We thus need
to swap $\chi_i$ and $\chi_{i-1}$ in $\v{\chi}$. Moreover, the
intersection $\tilde{\nu}_{i-1}$ and $\tilde{\nu}_{i+1}$ become outdated due
to the swap in positions. We
locate their indices using the back-pointers on the priority queue and take
them out of the priority queue. (Note that neither of them must reside in the
queue since we might have encountered these values earlier during the scanning
process.) Since the lines changed their order, there might be newly formed
intersections looming ahead. These intersections,
$\tilde{\nu}_{i-1}$ and $\tilde{\nu}_{i+1}$, are computed and added to
the queue if they are larger than $\tilde{\nu}_i$. In Fig.~6 we illustrate
the process of updating the priority queue when processing the current
element at the top of the queue.

When either $\chi_i=0$ or $\chi_{i-1}=0$
extra measures need to be taken since we encounter an intersection of the
line $\lll_0$ with one of the lines $\lll_{\pm j}$ for $j\ge 1$. First, we
add a new line segment to $\mu(\nu)$ originating at
$(\tilde{\nu}_i,\tilde{\mu}_i)$. We next update $\s{M},\s{U},\s{Q}$ as
required by the intersection in the same way the update is performed when
conducting the local homotopy tracking. Now that we have the updated slope
of $\lll_0$ available, we can proceed and calculate potential new intersections
with the neighboring (above and below) lines and add them to the
priority queue in case their value is larger than $\tilde{\nu}_i$.

The global homotopy tracking finishes when the queue is empty. It requires
$O(n)$ storage and $O(n^2\log(n))$ operations since each insertion and
deletion from the queue can be performed in $O(\log(n))$ steps. Clearly,
if the number of line segments constituting $\mu(\nu)$ is greater than
$n\log(n)$ (recall that the upper bound is $O(n^2)$) then the global
homotopy procedure is faster than the local one.

\paragraph{Example} Since the global tracking process is somewhat complex,
we now give a concrete example which illustrates the process.
We examine the following toy problem setting,
$$\v{u}=\left(\frac{1}{2}, \frac{1}{8}, \frac{1}{12}\right) ~ ~ , ~ ~
  \v{q}=\left(\frac{1}{4}, \frac{1}{3}, \frac{1}{36}\right) ~ ~ , ~ ~
  \v{m}=\left(1, 2, 3\right) ~ ~ . $$
At the start we set $\nu=0$, $\s{M}=0$, $\s{U}=\s{Q}=1$, and
$\v{\chi}=(-3,-2,-1,0,1,2,3)$. The queue contains the following pairs
of values and back-indices,
$$\left({12}/{7}, 6\right) ~ ~ , ~ ~
  \left({36}/{13},2\right) ~ ~ , ~ ~
  \left(4,4\right) ~ ~ . $$

\noindent
Next we pull $({12}/{7},6)$ from the front of the queue and set
$\nu\leftarrow{12}/{7}$. We reorder $\v{\chi}$ and set
$\v{\chi}\leftarrow(-3,-2,-1,0,1,3,2)$. Last for this intersection point,
we remove $({12}/{7},6)$ and add a newly found intersection, so that the
priority queue becomes,
$$({36}/{13},2) ~ ~ , ~ ~ (4,4) ~ ~ , ~ ~ (60,5) ~ ~ .$$

\iffalse
Updating
$\vec{U}\leftarrow\{\frac{1}{4},\frac{1}{2},1,1,\frac{1}{2},\frac{1}{4}\}$,
$\vec{Q}\leftarrow\{\frac{1}{12},\frac{3}{4},1,1,\frac{3}{4},\frac{2}{3}\}$.
By coincidence the changes to $\vec{U}$ canceled each other.
\fi

\noindent
We next pull $({36}/{13}, 2)$ out of the queue, thus setting
$\nu\leftarrow{36}/{13}$. The last removal results in the following line
ordering $\v{\chi}\leftarrow(-3,-1,-2,0,1,3,2)$. We also identify a new
line intersection and the queue becomes,
$$(4,4) ~ ~ , ~ ~ ({24}/{5},3) ~ ~ , ~ ~ (60,5) ~ ~ .$$

\iffalse
Updating
$\vec{U}\leftarrow\{\frac{1}{4},\frac{3}{4},1,1,\frac{1}{2},\frac{1}{4}\}$,
$\vec{Q}\leftarrow\{\frac{1}{12},\frac{1}{3},1,1,\frac{3}{4},\frac{2}{3}\}$.
\fi

\noindent
We then remove $(4,4)$ from the front of the queue, setting
$\nu\leftarrow 4$, the line ordering
$\v{\chi}\leftarrow(-3,-1,-2,1,0,3,2)$, and the queue
to contain the following pairs
$$({60}/{13},3) ~ ~ , ~ ~ (15,5) ~ ~ .$$
One of the lines at the current intersection ($\nu=4$) is $\lll_0$ we thus
perform the update
$$\s{M}\leftarrow1 ~ ~ , ~ ~
  \s{U}\leftarrow\frac{1}{2} ~ ~ , ~ ~
  \s{Q}\leftarrow\frac{3}{4} ~ ~ . $$
The global homotopy tracking continues the above schemes and goes through
numerous more line intersections.

\end{document}